\documentclass[twoside]{article}

\usepackage[accepted]{aistats2022}

\usepackage[round]{natbib}

\usepackage{hyperref}       
\usepackage{url}            
\usepackage{booktabs}       
\usepackage{amsfonts}       
\usepackage{xcolor}         

\usepackage{algorithm}
\usepackage{algpseudocode}
\usepackage{amsmath,amsfonts,amsthm,amssymb}
\usepackage{enumitem}
\usepackage{graphicx}
\usepackage[most]{tcolorbox}
\usepackage{wrapfig}
\usepackage{caption}
\usepackage{multicol}
\usepackage{multirow}

\hypersetup{
    colorlinks=true,
    linkcolor=blue,
    citecolor=blue,
    urlcolor=magenta
}

\newtheorem{prop}{Proposition}

\newtheorem{manualtheoreminner}{Proposition}
\newenvironment{manualtheorem}[1]{%
  \renewcommand\themanualtheoreminner{#1}%
  \manualtheoreminner
}{\endmanualtheoreminner}

\newcommand{\shams}[1]{{\color{black}{#1}}}

\begin{document}

\runningauthor{S.S. Azam, T. Kim, S. Hosseinalipour, C. Joe-Wong, S. Bagchi, C. Brinton}

\twocolumn[

\aistatstitle{Can we Generalize and Distribute Private Representation Learning?}

\vspace{-4mm}
\aistatsauthor{ Sheikh Shams Azam \And Taejin Kim \And Seyyedali Hosseinalipour}
\aistatsaddress{ School of ECE\\Purdue Univeristy\\IN, USA \And Department of ECE\\Carnegie Mellon University\\CA, USA \And School of ECE\\Purdue Univeristy\\IN, USA}
\aistatsauthor{Carlee Joe-Wong \And Saurabh Bagchi \And Christopher Brinton}
\aistatsaddress{ Department of ECE\\Carnegie Mellon University\\CA, USA \And School of ECE\\Purdue Univeristy\\IN, USA \And School of ECE\\Purdue Univeristy\\IN, USA}
]


\begin{abstract}

We study the problem of learning representations that are private yet informative, i.e., provide information about intended ``ally'' targets while hiding sensitive ``adversary'' attributes. We propose Exclusion-Inclusion Generative Adversarial Network (EIGAN), a generalized private representation learning (PRL) architecture that accounts for multiple ally and adversary attributes \shams{unlike existing PRL solutions. While centrally-aggregated dataset is a prerequisite for most PRL techniques,} data in real-world is often siloed across multiple distributed nodes unwilling to share the raw data because of privacy concerns. We address this practical constraint by developing D-EIGAN, the first distributed PRL method that learns representations at each node without transmitting the source data. We theoretically analyze the behavior of adversaries under the optimal EIGAN and D-EIGAN encoders and the impact of dependencies among ally and adversary tasks on the optimization objective. Our experiments on various datasets demonstrate the advantages of EIGAN in terms of performance, robustness, and scalability. In particular, EIGAN outperforms the previous state-of-the-art by a significant accuracy margin (47\% improvement), and D-EIGAN's performance is \shams{consistently on par with EIGAN under different network settings.}
\end{abstract}


\section{INTRODUCTION}
\label{sec:intro}

Training machine learning (ML) models often requires sharing data among multiple parties, e.g., cloud services aggregating data from multiple users to learn a global model. Such data sharing naturally raises concerns \citep{chakraborty2013framework,saleheen2016msieve} about exposing sensitive user attributes in datasets. \shams{It is thus imperative that both data aggregators and users engage in/propose procedures that minimize leakage of sensitive information.}

A widely used technique for obfuscating sensitive attributes in data is \shams{context-agnostic noise injection (e.g. Laplace mechanism)} \citep{diff_priv}, that introduces additive noise into a dataset to provide membership security \citep{li2007t}.
However, noise injection can impact ML training and inference significantly \citep{tossou2017achieving}. This makes such context-agnostic techniques unsuitable in scenarios where only a few attributes 
need to be concealed. For example, upon sharing patient data for preventive healthcare \citep{henry2015targeted, azam2019cascadenet}, both privacy (e.g., gender anonymization) and predictivity (e.g., accurate diagnosis) are desirable.

These drawbacks of context-agnostic privacy measures motivate private representation learning (PRL) \citep{privacy_gan}, which exploits the knowledge of sensitive attributes in a dataset. PRL considers privacy and predictivity as joint (and possibly competing) objectives, and learns a transformation on the data that balances the goals of (i) obfuscating sensitive attributes of interest to an ``adversary'' (adv.) while (ii) preserving predictivity on intended targets for an ``ally'' \citep{ganin2016domain}.

Conventionally, the literature on PRL assumes the existence of a single sensitive attribute and a central dataset \citep{privacy_gan, advr_info_leakage, advr_rep_learn, bertran2019adversarially}. However, most real-world datasets have multiple sensitive attributes and are collected across multiple distributed nodes. 
Healthcare records, for example, are (i) spread across hospitals in different regions,
(ii) consist of potentially multiple sensitive attributes, such as mental health, gender, ethnicity, etc., and (iii) may have varying notions of privacy that vary from one region to another, e.g., 
while in Europe racial/ethnic origin are considered as sensitive information (as per GDPR), in USA they are not (as per HIPAA).
These challenges call for a {\em generalized and distributed} PRL methodology that takes into account multiple sensitive attributes, trains on data distributed across nodes, and learns representations that incorporate the privacy/predictivity goals of each node.
Communication-efficiency is also a key objective in distributed learning, particularly 
when it is being deployed 
in network settings where nodes are restricted to communicate over limited-bandwidth links~\citep{hosseinalipour2020multi,chatterjee2020context}, e.g., remote health analytics across user devices \citep{dimitrov2016medical}.

\begin{figure}[t!]
\begin{center}
\centerline{\includegraphics[width=1.0\columnwidth]{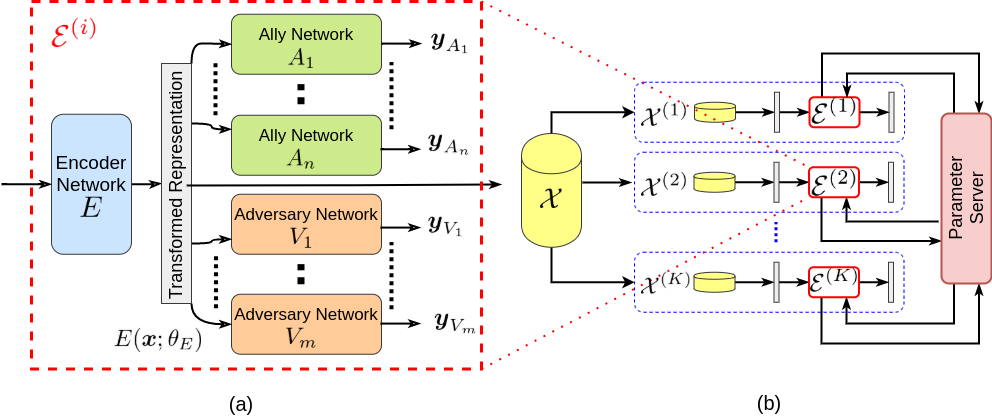}}
    \vskip -0.1in
\caption{\small{(a) Architecture of a single EIGAN node, consisting of an encoder, $n$ ally, and $m$ adversary networks. (b) D-EIGAN system for distributed EIGAN training, consisting of $K$ different EIGAN nodes, each with their own subset of the full dataset. The nodes must coordinate their local encodings via a parameter server.}}
\label{fig:dist-eigan}
\end{center}
\vskip -0.25in
\end{figure}

In this paper, we propose a novel PRL architecture
called {\bf {\em Exclusion-Inclusion Generative Adversarial Network (EIGAN)}}, which addresses the aforementioned challenges.
EIGAN is a generalized PRL technique designed to generate encodings ``inclusive'' of signals that are of utility to a set of allies, while ``exclusive'' of signals that can be used by adversaries to recover sensitive attributes.
Further, to address the privacy vulnerabilities of pooling raw data, we develop {\bf {\em D-EIGAN}} (for Distributed-EIGAN), 
where multiple EIGAN nodes 
train encoders on their local datasets
and synchronize their model parameters periodically, as depicted in Fig. \ref{fig:dist-eigan}. D-EIGAN implements distributed training without noticeable model degradation compared to the centralized EIGAN, while accounting for realistic factors of communication constraints and non-i.i.d data distributions across nodes.

\textbf{Related work.}
Recent works in PRL~\citep{privacy_gan,advr_info_leakage,advr_rep_learn,raval2017protecting,wu2018towards} have proposed centralized architectures that jointly maximize the loss in predicting sensitive attributes 
while minimizing the loss of target task prediction. Specifically,~\citet{privacy_gan} proposed a three-network encoder-ally-adversary architecture and showed that the achievable tradeoff between the two objectives is better than that provided by DP. In~\citet{advr_info_leakage}, the problem was formulated as a non-zero-sum game between the three networks to minimize information leakage in encoded image representations. \citet{advr_rep_learn} experimentally outperform  \citet{advr_info_leakage,louizos2015variational,xie2017controllable,pmlr-v28-zemel13} using a minimax optimization among three networks, and derive its closed-form solution when the networks are linear maps. We demonstrate that EIGAN converges to the optimal performance obtained by these closed form solutions. \shams{However, unlike the closed form solution in \citet{advr_rep_learn}, EIGAN can be extended to account for multiple ally/adv. attributes.} Furthermore, EIGAN has computational advantage over \citet{advr_rep_learn} as it does not depend on matrix inversions\shams{, and thus can work with higher dimensional data}.

\begin{table}[t]
    \centering
    \resizebox{\columnwidth}{!}{\begin{tabular}{| c | c  c | c  c |}
      \hline
      & \multicolumn{2}{c|}{Adult Dataset}& \multicolumn{2}{c|}{Facescrub Dataset}\\
      \hline
      \multirow{2}{*}{\textbf{Objective}} & \textbf{Ally} & \textbf{Adversary} & \textbf{Ally} & \textbf{Adversary} \\
      & (identity) & (gender) & (income) & (gender) \\
      \hline\hline
      Unencoded & 0.85 & 0.85 & 0.98 & 0.99 \\
      \hline
      Linear-ARL & \textbf{0.84} & 0.67 & - & - \\
      Kernel-ARL & \textbf{0.84} & 0.67 & - & - \\
      Bertran-PRL & 0.82 & 0.67 & \textbf{0.56} & 0.68\\
      EIGAN & \textbf{0.84} & 0.67 & \textbf{0.82} & 0.68 \\ 
      \hline
      \multirow{2}{*}{\% Improv.} & Matches closed & Controlled & \multirow{2}{*}{47.01\%} & Controlled \\
      & form solution & to be equal & & to be equal \\
      \hline
    \end{tabular}}
    \vspace{0.01in}
    \caption{\small{Performance comparison between EIGAN, \citet{advr_rep_learn} (Linear-ARL, Kernel-ARL), and \citet{bertran2019adversarially} (Bertran-PRL) on the Adult \& FaceScrub datasets considered in those works. 
    For the same adv. performance, EIGAN obtains a notable improvement over \citet{bertran2019adversarially} (ally improvement of $47.01\%$). It also reaches the optimal closed form solution of \citet{advr_rep_learn}.
    }}
    \label{table:bertran_compare}
  \vspace{-0.05in}
\end{table}

Other PRL works take an information-theoretic approach. \citet{bertran2019adversarially} view PRL as minimization of the utility lost in the learned representation, subject to an upper bound on mutual information between the output representation and the sensitive attribute. Similarly, \citet{ppan} formulate the minimax problem in terms of KL-divergence. EIGAN, on the other hand, considers a cross-entropy PRL formulation, which promotes interpretability and training stability over multiple objectives (discussed in Section~\ref{ssec:c-arch}).
Furthermore, our experiments show that EIGAN significantly outperforms the state-of-the-art \citep{bertran2019adversarially} in the single ally/adversary case. {\em Distinct from all prior work in PRL, we consider multiple sensitive attributes and distributed learning.} 

There are two other related directions in adversarial learning. One addresses privacy-preservation through synthetic data generation \citep{pategan, dpgan}, which differs from EIGAN's goal of learning a transformation. The other is fair representation learning \citep{oh2016faceless, edwards2015censoring, kusner2017counterfactual}, which seeks to learn intrinsically fair representations that promote demographic parity on a single attribute \citep{madras2018learning}.

\textbf{Contributions.}
Our main contributions are:
\begin{enumerate}[leftmargin=4mm]
\item We introduce EIGAN 
(Section~\ref{ssec:c-arch}), generalizing PRL to 
account for multiple target and sensitive attributes.
We prove that EIGAN's encoder utility is maximized if the adversary outputs follow a uniform distribution, and consider the effect of correlations between ally and adversary objectives (Prop. \ref{prop:uniform}).
\item To the best of our knowledge, D-EIGAN (Section~\ref{ssec:d-arch}) is the first technique for distributed training of PRL models. We show that when the nodes engaged in the training possess \shams{independent and identically distributed} (i.i.d) datasets, the objective of D-EIGAN exhibits similar properties to EIGAN (Prop.~\ref{prop:distributed_iid}).
\item Our experiments (Section~\ref{sec:results}) reveal that EIGAN significantly outperforms the state-of-the-art in PRL (Table~\ref{table:bertran_compare}, Fig.~\ref{fig:gaussian_bert_ally}) and is robust to the choice of adversary architectures (Table~\ref{table:robust_compare}). 
We also demonstrate that D-EIGAN matches the performance of EIGAN even as the number of nodes increases (Fig.~\ref{fig:distributed_num_nodes}), and is robust even when nodes have different objectives (Fig.~\ref{fig:distributed_eigan_mimic}). We further show the resilience of D-EIGAN to non-i.i.d data distributions across nodes, and under communication restrictions that require partial parameter sharing and delayed model aggregations in the system (Fig.~\ref{fig:distributed_control_params}).
\end{enumerate}


\section{EIGAN MODEL}
\label{sec:model}

\textbf{Overview.} Our PRL methodology consists of two phases: training and testing. In the training phase, EIGAN -- knowing the sensitive/target labels of interest to adversary/ally on the train dataset -- aims to learn the encoder by simulating allies and adversaries. Each of the allies, adversaries, and encoder independently maximize their own utilities by updating their local model parameters. The selfish maximization by each player naturally leads to the minimax optimization in \eqref{eqn:minimax_game2}. In the testing phase, the test data undergoes a transformation through the trained encoder. The transformed data is used for conventional training and inference by the actual allies and adversaries on their respective tasks of interest.

In Section~\ref{ssec:c-arch}, we present the EIGAN formulation for centralized model training, and derive properties of the solution. Then, we extend it to the distributed learning case, D-EIGAN, in Section~\ref{ssec:d-arch}.  
Refer to Appendix.~\ref{supp:proofs} for the proofs of the propositions.

\subsection{EIGAN: Centralized Model Architecture}
\label{ssec:c-arch}

We first consider a system consisting of $n$ allies, indexed $A_1,...,A_n$; and $m$ adversaries, indexed $V_1,...,V_m$. Ally $A_i$ is characterized by model parameters $\theta_{A_{i}}$  and a set of target attributes/labels $Y_{A_{i}}$ drawn from  distribution $\mathcal{Y}_{_{A_{i}}}$. $A_i$ aims to associate each input sample with its corresponding target attribute in $Y_{A_{i}}$. Similarly, adversary $V_j$ parameterized by $\theta_{V_{j}}$ wishes to associate input samples with a set of (known) sensitive attributes/labels $Y_{V_{j}}$ following distribution $\mathcal{Y}_{V_j}$. 

The goal of EIGAN is to learn an encoder $E$ parameterized by $\theta_E$ that maximizes the performance of $A_1,...,A_n$ while minimizing the performance of $V_1,...,V_m$. The encoder uses a centrally-located dataset $\mathcal{X}$ 
consisting of $N$ samples, where each sample 
is represented  as a $d$-dimensional feature vector $\boldsymbol{x}_j \in \mathbb{R}^d$, $j = 1,...,N$.
We let $E(\boldsymbol{x};\theta_E)$ denote the output of the encoder for a data sample $\boldsymbol{x}$ realized via the parameters $\theta_E$. 
$E(\boldsymbol{x};\theta_E): \mathbb{R}^d \rightarrow \mathbb{R}^l$ is in general a non-linear differentiable function (e.g., a neural network), 
where $l$ is the dimension of the representation output by the encoder, and typically $l \leq d$.

For $\boldsymbol{x} \hspace{-1mm}\in\hspace{-1mm} \mathcal{X}$, the encoded representation $E(\boldsymbol{x};\theta_E\hspace{-0.5mm})$ is what the allies $A_1,..,A_n$ and adversaries $V_1,..,V_m$ are provided with for their tasks, as depicted in Fig.~\ref{fig:dist-eigan}(a). We quantify the utilities of the allies and adversaries as: 
\vspace{-3mm}

\begin{align}\label{eq:lossFunctions}
\begin{split}
  &\hspace{-0.5mm}u_{_{A_i}} \hspace{-0.8mm}= \mathbb{E}_{{Y}\sim \mathcal{Y}_{_{A_{i}}}} \hspace{-0.8mm}\left[ \log\left(p_{_{A_i}}\hspace{-0.8mm}\left({Y}|E(\mathcal{X}; \theta_E)\right)\right) \right]\hspace{-0.5mm},\hspace{-0.5mm}1\hspace{-0.5mm}\leq\hspace{-0.5mm} i\hspace{-0.5mm}\leq\hspace{-0.5mm} n, \\
   &\hspace{-0.5mm}u_{_{V_j}} \hspace{-0.8mm}= \mathbb{E}_{Y\sim \mathcal{Y}_{_{{V}_j}}} \hspace{-0.8mm}\left[ \log\left(p_{_{V_j}}\hspace{-0.8mm}\left({Y}|E(\mathcal{X}; \theta_E)\right)\right) \right]\hspace{-0.5mm},\hspace{-0.5mm}1\hspace{-0.5mm}\leq \hspace{-0.5mm}j\hspace{-0.5mm}\leq \hspace{-0.5mm}m,  
\end{split}
\end{align}

\noindent where $p_{_{A_i}}\left({Y}|E(\mathcal{X}; \theta_E)\right)$ and $p_{_{V_j}}\left({Y}|E(\mathcal{X}; \theta_E)\right)$ denote the probabilities of successful inference of target labels $Y \sim \mathcal{Y}_{A_i}$ and sensitive labels $Y \sim \mathcal{Y}_{V_j}$ for ally $A_i$ and adversary $V_j$, respectively, over the outputs that the encoder $E$ provides for the dataset $\mathcal{X}$. This leads to our minimax game among three types of players, in which two (the encoder and allies) are colluding against the third (the adversary). Specifically, we formulate the optimization problem:
\begin{equation}
\min_{\theta_{V}=\{\theta_{V_{j}}\}_{j=1}^{m}}~~~ \max_{\theta_{E},\theta_{A}=\{\theta_{A_{i}}\}_{i=1}^{n}} U(\theta_{E}, \theta_{A}, \theta_{V}),
\label{eqn:minimax_game2}
\end{equation}
where
\begin{equation}
U(\theta_{E},\theta_{A}, \theta_{V})=\sum_{i=1}^{n} \alpha_{_{A_i}} u_{_{A_i}} - \sum_{j=1}^{m}  \alpha_{_{V_j}} u_{_{V_j}}.
\label{eqn:minimax_game3}
\end{equation}
Here, $\smash{\alpha_{_{A_i}}, \alpha_{_{V_j}} > 0}$ denote normalized importance parameters placed on each objective such that $\smash{\sum_{i=1}^{n} \alpha_{A_i} + \sum_{j=1}^{m} \alpha_{V_j} = 1}$. Similar to the encoder, we assume that the ally and adversary are non-linear, differentiable functions. The encoder in \eqref{eqn:minimax_game2} seeks to maximize the achievable utility of the allies while minimizing those of the adversaries, operating in conjunction with the allies in the inner max layer of \eqref{eqn:minimax_game2}. The adversaries then operate on the encoder result in the outer min layer, where each adversary $V_j$ aims to maximize its utility $u_{_{V_j}}$ by updating $\theta_{_{V_j}}$, as it cannot access other ally/adversary's parameters directly.

It is worth noting that, similar to the formulation based on mutual information in \cite{bertran2019adversarially}, our analysis on the expected posterior distribution of the predictions in EIGAN map directly to interpretable metrics such as accuracy \citep{bassily2018learners} and generalization error \citep{feder1994entropy}, instead of the worst case guarantees provided by context-agnostic privacy frameworks such as DP.

Intuitively, the encoder will attempt to diminish the adversary predictions to a random guess, i.e., to a uniform distribution over its target labels \citep{adv_crypto}. However, this may be difficult to achieve when the interests of the allies and adversaries are related, which makes the weights $\alpha_{_{A_i}}, \alpha_{_{V_j}}$ important to the minimax solution in \eqref{eqn:minimax_game2} formalized in the proposition below:

\begin{prop}\label{prop:uniform}
Let $\mathcal{O}$ denote the set of all $(i,j)$ pairs of allies $A_i$ and adversaries $V_j$ 
for which ${Y}_{A_i} \cap {Y}_{V_j} \neq \emptyset$, i.e., overlapping interests. Given a fixed encoder $E$ in EIGAN architecture, if $\mathcal{O}=\emptyset$,  the overall score in \eqref{eqn:minimax_game2} is maximized when the adversaries' output predictions follow a uniform distribution. On the other hand, if  $\mathcal{O} \neq \emptyset$, then for each overlapping label, the architecture proposed by \eqref{eqn:minimax_game2} considers the utility of the attributes that have the higher importance weight, i.e., $A_i$ if $\alpha_{A_i} > \alpha_{V_j}$ and $V_j$ if $\alpha_{A_i} < \alpha_{V_j}$. 
\end{prop}

Prop. \ref{prop:uniform} shows that given an encoded representation, if the allies and adversaries possess non-overlapping interests, then a uniform prediction distribution among the sensitive parameters of interest to the adversaries is adopted by the optimal solution. In Appendix.~\ref{suppsub:eigan-mnist}, we consider an experiment with such overlapping interests and equal importance weights, and find that EIGAN is unable to balance the objectives.

In practice, coincidental overlaps between ally and adversary interests would be relatively rare, but could nonetheless occur.
In such cases, EIGAN must balance predictivity and privacy, which leads to different ally and adversary outputs described in Prop.~\ref{prop:uniform}. We further analyze EIGAN's characteristics when there is a linear relationship between the target distribution of an ally and an adversary (see Prop.~\ref{prop:nonuniform} in Appendix.~\ref{suppsub:proof-2}).

\textbf{Model training.} We train the encoder and the allies/adversaries in EIGAN by alternately updating their parameters using stochastic gradient descent (SGD) to minimize their cross-entropy (CE) loss. For the encoder, we define the CE-loss $\mathcal{L}_E$ for a single training instance as a weighted combination of the predictive capability of the allies and adversaries as
\vspace{-3mm}

{\small
\begin{equation}
    {\mathcal{L}}_E = \sum_{i=1}^{n}\underbrace{- \langle\boldsymbol{y}_{A_i}, \log \hat{\boldsymbol{y}}_{A_i}\rangle}_{\substack{\text{loss of ally } A_i \text{, } {\mathcal{L}}_{A_i}}} - \alpha \cdot \sum_{j=1}^{m} \underbrace{\hspace{-5mm}-\langle\boldsymbol{y}_{V_j}, \log \hat{\boldsymbol{y}}_{V_j}\rangle}_{\substack{\text{loss of adversary } V_j \text{, } {\mathcal{L}}_{V_j}}}\hspace{-5mm},
    \label{eqn:log_loss}
\end{equation}}

\noindent where $\langle.,.\rangle$ denotes inner product, and $\log$ is applied element-wise. $\boldsymbol{y}_{A_i}$ and $\boldsymbol{y}_{V_j}$ are the binary vector representations of the true class labels for ally $A_i$ and adversary $V_j$, respectively, while $\hat{\boldsymbol{y}}_{A_i}$ and $\hat{\boldsymbol{y}}_{V_i}$ are the vectors of soft predictions (i.e., probabilities) for each class. Here, we have made the simplifications $\alpha_{_{A_i}} = \alpha / n \; \forall i$ and $\alpha_{_{V_{j}}} = (1 - \alpha) / m \; \forall j$, where $\alpha \in (0,1)$ is tuned to emphasize either predictivity (higher $\alpha$) or privacy (lower $\alpha$). It can be seen that the minimization of loss $\mathcal{L}_E$ is equivalent to the maximization of utility defined by \eqref{eqn:minimax_game3}. In each epoch, we average $\mathcal{L}_E$ over a minibatch of size $J$ to obtain an estimate of \eqref{eq:lossFunctions}, and update $\theta_E$ based on the gradient. Then, we update the $\theta_{A_i}$ and $\theta_{V_i}$ according to~\eqref{eqn:log_loss}. See Alg.~\ref{alg:eigan} in Appendix.~\ref{supp:algorithms} for details.

\textbf{Loss consideration.} Alternative objectives to \eqref{eqn:log_loss} exist in PRL literature. In particular, recent works \citep{gap,ppan, bertran2019adversarially} formulate the adversarial loss using KL divergence.
We choose CE-loss over KL-divergence based on the fact that KL divergence fails to give meaningful value under disjoint distributions \citep{wgan}. \shams{Also, our CE-loss formulation is unconstrained as opposed to KL-divergence formulation which is a Lagrangian dual of the constrained formulation \citep{bertran2019adversarially}. As analyzed in Prop.~\ref{prop:uniform}, our formulation naturally pushes adversary prediction towards uniform distribution, however, the same does not hold for the constrained formulation.} Our results in Table~\ref{table:bertran_compare} and Fig.~\ref{fig:gaussian_bert_ally} validate our formulation choice. We show consistent improvements over the state-of-the-art \citep{bertran2019adversarially} that uses KL divergence.

\begin{figure}[t!]
\begin{center}
    \centerline{\includegraphics[width=1.0\columnwidth]{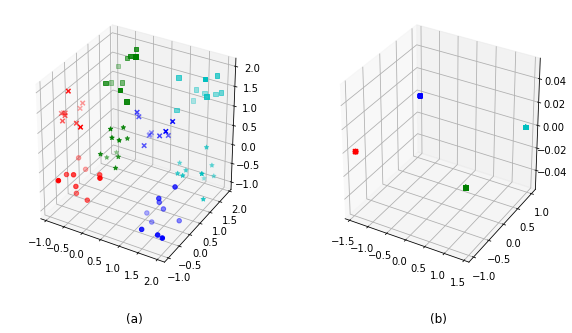}}
    \vskip -0.1in
    \captionof{figure}{\small{(a) Synthetic dataset with eight groups of points, two allies, and one adversary. The allies are interested in separating the color pairs (the two horizontal axes), and the adversary is interested in classifying shapes (the vertical axis). (b) EIGAN's encoding has collapsed the adversary dimension while preserving the allies.}}
    \label{fig:octant_advr}
\end{center}
\vskip -0.25in
\end{figure}

\textbf{Visual demonstration.} Fig.~\ref{fig:octant_advr} is a visual demonstration of EIGAN's trained representation on a synthetic dataset. There are two allies $A_1$ and $A_2$ which are each interested in separating data points along one of the horizontal axes, and an adversary $V$ that is interested in separation along the vertical axis. We see in (b) that the EIGAN encoding collapses the data along the vertical axis while retaining separability in the other two dimensions. Other illustrations are given in Appendix.~\ref{supp:poc-vis}.

\subsection{D-EIGAN: Distributed EIGAN Model}
\label{ssec:d-arch}

The distributed setting for EIGAN (D-EIGAN) is depicted in Fig.~\ref{fig:dist-eigan}(b). There are $K$ nodes in the system, denoted $\mathcal{E}^{(1)},...,\mathcal{E}^{(K)}$, and a parameter server for model synchronization. Each node $\mathcal{E}^{(k)}$ has a set of allies, denoted $A^{(k)}_{1},...,A^{(k)}_{n(k)}$ with target label sets $Y_{A^{(k)}} = \{Y_{A^{(k)}_{1}},..., Y_{A^{(k)}_{n(k)}} \}$, a set of adversaries, denoted $V^{(k)}_{1},...,V^{(k)}_{m(k)}$ with target sets $Y_{V^{(k)}} = \{Y_{V^{(k)}_{1}},...,Y_{V^{(k)}_{m(k)}}\}$, and a subset $\mathcal{X}_k \subset \mathcal{X}$ of $N_k$ datapoints from the overall dataset $\mathcal{X}$ of $N$ samples. These local datasets are in general non-overlapping, and may differ in size. While the specific allies and adversaries may differ at each node, the goal is to train encoder models that maximize all allies' and minimizes all adversaries' performances, so that the encodings are meaningful throughout the system. Since sharing the raw datasets could potentially leak sensitive information, each node $\mathcal{E}^{(k)}$ will train its own local encoder $E^{(k)}(\boldsymbol{x}; \theta_{E^{(k)}})$, and the server in Fig.~\ref{fig:dist-eigan}(b) will periodically aggregate the locally-trained models.

The utility function for node $\mathcal{E}^{(k)}$ is defined as
\vspace{-3mm}

{\footnotesize\begin{equation}\label{eq:UtilityDist}
\hspace{-1mm}U^{(k)}\hspace{-0.5mm}(\theta_{E^{(k)}}\hspace{-0.5mm},\hspace{-0.5mm}\theta_{A^{(k)}}\hspace{-0.5mm},\hspace{-0.5mm}\theta_{V^{(k)}}\hspace{-0.5mm})\hspace{-0.5mm}=\hspace{-0.5mm}\sum_{i=1}^{n} \hspace{-0.5mm} \alpha_{{A^{(k)}_i}} u_{A^{(k)}_i}\hspace{-0.5mm}- \hspace{-0.5mm}\sum_{j=1}^{m} \hspace{-0.5mm} \alpha_{V^{(k)}_j} u_{V^{(k)}_j},
\end{equation}}

\vspace{-1mm}
\noindent where $\theta_{A^{(k)}}=\big\{\theta_{A^{(k)}_{i}}\big\}_{i=1}^{n(k)}$ and $\theta_{V^{(k)}}=\big\{\theta_{V^{(k)}_{j}}\big\}_{j=1}^{m(k)}$ denote the sets of ally and adversary parameters at node $\mathcal{E}^{(k)}$, and $u_{A^{(k)}_i}, u_{V^{(k)}_j}$ denote the utility functions of $A^{(k)}_i, V^{(k)}_j$ defined analogously to~\eqref{eq:lossFunctions}. $\alpha_{A^{(k)}_i}, \alpha_{V^{(k)}_j} > 0$ denote the normalized importance parameters for node $\mathcal{E}^{(k)}$, where $\sum_{i=1}^{n} \alpha_{A^{(k)}_i} + \sum_{j=1}^{m} \alpha_{V^{(k)}_j} = 1$. This leads to the following minimax game for the distributed case:
\vspace{-7mm}

{
\vspace{-2mm}
\begin{equation}\label{eqn:minimax_game4}
\begin{aligned}
        \min_{\mathcal{S}_V}~ \max_{\mathcal{S}_E,\mathcal{S}_A} ~ & \frac{1}{K} \sum_{k=1}^{K} U^{(k)}(\theta_{E^{(k)}},\theta_{A^{(k)}},\theta_{V^{(k)}}) \\ \textrm{s.t.}
   ~~~~~~ & \theta_{E^{(k)}}=\theta_{E^{(k')}},~k \neq k', 1\leq k,k'\leq K,
\end{aligned}
\end{equation}
}

\noindent where $\smash{\footnotesize\mathcal{S}_V = \left\{\theta_{V^{(k)}}\right\}_{k=1}^{K} ,\mathcal{S}_E = \left\{\theta_{E^{(k)}}\right\}_{k=1}^{K}}$, and $\smash{\mathcal{S}_A=\left\{\theta_{A^{(k)}}\right\}_{k=1}^{K}}$.
The constraint in~\eqref{eqn:minimax_game4} ensures that the optimal encoder is the same across all nodes, even though each node may have different allies and adversaries. In this way, an encoded datapoint $E^{(k)}(\boldsymbol{x}; \theta_E)$ at node $k$ could be transferred to another node $k'$ and applied to a task $A^{(k')}_{i}$ privately, e.g., for anonymized user data sharing during single sign-ons.

\textbf{Distributed model training.} While solving \eqref{eqn:minimax_game4} in a distributed manner, D-EIGAN learns both a global model and personalized local models (allies and adversaries) \citep{smith2017federated}, unlike standard Federated Learning (FL).

Our algorithm consists of two iterative steps. The first is \textit{local update}: each $\mathcal{E}^{(k)}$ conducts a series of $\delta$ SGD iterations. For each minibatch in SGD, training proceeds as in the centralized case, with the encoder, allies', and adversaries' parameters updated via SGD to minimize the CE-losses $\mathcal{L}_E^{(k)}$, $\mathcal{L}_{A_i}^{(k)}$, and $\mathcal{L}_{V_j}^{(k)}$ defined as in \eqref{eqn:log_loss} but in this case for each node. The second step is \textit{global aggregation}, in which each $\mathcal{E}^{(k)}$ uploads its locally-trained encoder to the parameter server to construct a global version, after every $\delta$ SGD iterations. We introduce a sparsification technique here in which each node selects a fraction $\phi$ of its parameters at random to upload for each aggregation. Letting $\mathcal{Q}_k$ be the indices chosen by $\mathcal{E}^{(k)}$, then the vector recovered at the server is $\tilde{\theta}_{E^{(k)}}$, where $\tilde{\theta}_{E^{(k)}}(q) = {\theta}_{E^{(k)}}(q)$ if $q \in \mathcal{Q}_k$ and $0$ otherwise. With this, the global aggregation becomes the weighted average $\theta_E = \sum_k \frac{N_k}{N} \tilde{\theta}_{E^{(k)}}$. Then, the server also selects a fraction $\phi$ of indices at random to synchronize each node $k$ with on the downlink. Letting $\mathcal{Q}$ be these indices, each node $k$ sets $\theta_{E^{(k)}}(q) = \theta_E(q)$ if $q \in \mathcal{Q}$, and makes no change to the $q$th parameter otherwise. The pseudo-code of the training procedure is given in Alg.~\ref{alg:d-eigan}.

The synchronization frequency $\delta$ and sparsification factor $\phi$ are directly related to the amount of data transferred through the system: as $\delta$ increases, uplink transfers to the server occur less frequently; as $\phi$ decreases, each uplink/downlink transmission requires fewer communication resources. This is an important consideration in networking applications where the nodes communicate over a resource-constrained channel \citep{wang2019adaptive, pmlr-v54-mcmahan17a}. Fractional parameter sharing, similar to pruning (both choose a subset of parameters), mimics the additive-noise DP mechanism \citep{huang2020privacy} on model weights, reducing associated leakage \citep{fredrikson2015model, shokri2017membership} to any untrusted entity with access to the system. We study the effect of $\delta$ and $\phi$ on D-EIGAN performance in Section~\ref{ssec:d-eigan-expt}.

\begin{algorithm}[t]
{\footnotesize 
   \caption{D-EIGAN training}
\label{alg:d-eigan}
\begin{algorithmic}[1]
  \Statex \textbf{\underline{Notation:}}
    \State $\theta_E$: global parameter vector
    \State $\mathcal{Q}_k$: uniformly random choice of indices at node $\mathcal{E}^{(k)}$
    \State $\tilde{\theta}_{E^{(k)}}$: parameter vector recovered at the server for encoder $E^{(k)}$, with its $q$th element denoted $\tilde{\theta}_{E^{(k)}}(q)$
    \State $\phi$: fraction of parameters shared
    \State $\delta$: number of epochs between aggregations
    \State ${(\cdot)_j}$: value for the $j$th minibatch
    \State $\mathcal{L}_{A^{(k)}_{i}}$ and $\mathcal{L}_{{V^{(k)}_{i}}}$: loss of ally $\small{A^{(k)}_{i}}$ and adversary $\small{V^{(k)}_{i}}$
    \State $\eta_{_E},\eta_{_A}$, and $\eta_{_V}$: learning rates 
   \Statex \textbf{\underline{Aggregation at Parameter Server:}}
   \State Initialize parameter $\theta_E$
   \For{each update round}
        \State Update parameter vector: $\theta_{E} \leftarrow \sum_{k=1}^K \frac{N_k}{N} \tilde{\theta}_{E^{(k)}}$
   \EndFor
   \Statex \textbf{\underline{Local Training at Node $\mathcal{E}^{(k)}$:}}
  \State Initialize $\big\{\theta_{A^{(k)}_{i}}\big\}_{i=1}^{n^{(k)}}$ and $\big\{\theta_{V^{(k)}_{j}}\big\}_{j=1}^{m^{(k)}}$
  \State Download initial $\theta_E$ from parameter server
   \For{number of training epochs}
        \State After $\delta$ epochs, update $\phi \cdot |\theta_{E^{(k)}}|$ chosen parameters 
        
        \Statex~~~~ from parameter server: $\theta_{E^{(k)}}(q)={\theta}_{E}(q)$ if $q\in \mathcal{Q}$
        
        \State Sample a minibatch $J$ from local dataset $\mathcal{X}_{k}$
        \State Update encoder: $\theta_{E^{(k)}} \leftarrow \theta_{E^{(k)}} - {\eta}_{_E} \cdot \nabla_{\theta_{E^{(k)}}} {\mathcal{L}}_{E^{(k)}}$
        \State Update ally/adversary parameters:
        
        \Statex~~~~~~~~ $\theta_{A^{(k)}_{i}}\leftarrow \theta_{A^{(k)}_{i}} - \eta_{_A} \cdot \nabla_{\theta_{A^{(k)}_{i}}} {\mathcal{L}}_{A^{(k)}_{i}}$, 
        \Statex~~~~~~~~ $\theta_{V^{(k)}_{i}} \leftarrow \theta_{V^{(k)}_{i}} - \eta_{_V} \cdot \nabla_{\theta_{V^{(k)}_{i}}} {\mathcal{L}}_{V^{(k)}_{i}}$
        \State After $\delta$ epochs, upload $\phi |\theta_{E^{(k)}}|$ encoder parameters:
        \Statex~~~~ $\tilde{\theta}_{E^{(k)}}(q)={\theta}_{E^{(k)}}(q)$ if $q\in \mathcal{Q}_k$, else $\tilde{\theta}_{E^{(k)}}(q)=0$
   \EndFor
\end{algorithmic}
}
\end{algorithm}

In D-EIGAN, the allies and adversaries may differ at each node, and each node trains an individual local encoder. Since the encoder parameters are globally synchronized, however, the local 
encoder implicitly trains using global union of allies/adversaries across nodes. In the case that the nodes have same objectives and i.i.d. datasets, we show that D-EIGAN yields the same properties as Prop.~\ref{prop:uniform}:

\begin{manualtheorem}{3}\label{prop:distributed_iid}
Given a set of fixed encoders in the D-EIGAN architecture, if all the nodes have the same number of allies and adversaries with the same sets of target labels ${Y}_{A^{(k)}}={Y}_{A^{(k')}}$ and ${Y}_{V^{(k)}}={Y}_{V^{(k')}}$, $1 \leq k, k' \leq K$, then Prop.~\ref{prop:uniform} holds for all the allies and adversaries belonging to different nodes if the local datasets at each node are i.i.d.
\end{manualtheorem}

When the nodes have different objectives, we show that the importance of each objective is proportional to the number of nodes implementing it; see Prop.~\ref{prop:dist_obj} in Appendix.~\ref{suppsub:proof-4}.


\section{EXPERIMENTAL EVALUATION/DISCUSSION}
\label{sec:results}

We now turn to an experimental evaluation of our methodology. 
We analyze EIGAN's convergence characteristics and compare its performance with relevant baselines in Section~\ref{ssec:c-eigan-expt}, and evaluate D-EIGAN compared to the centralized case and as the system characteristics change in Section~\ref{ssec:d-eigan-expt}.

\textbf{Datasets.} We consider datasets: MNIST \citep{mnist}, MIMIC-III \citep{mimiciii}, Adult \citep{Dua:2019},
and FaceScrub~\citep{facescrub_dataset}. MNIST consists of 60,000 handwritten digits with labels 0-9. MIMIC has medical information 
from hospitals with attributes, such as vitals and medication; we obtain a dataset consisting of 58,976 patients by joining multiple tables on patient IDs. Adult consists of 45,223 records extracted from the 1994 census data.
Facescrub is a dataset comprising over 22,000 images of celebrities with identity and gender labels. 

\textbf{Objectives.} In MIMIC, we consider survival (2-class) as the ally objective, and gender (2-class) and race (3-class) as adversary objectives. In the FaceScrub dataset, as in \cite{bertran2019adversarially}, the ally objective is user identity (200-class), and the adversary objective is gender (2-class). In MNIST, we consider whether a digit is even or odd (2-class) as the ally objective, and the label of the digit (10-class) as the adversary objective. In Adult, as in \cite{advr_rep_learn}, the ally objective is an annual income classification (more or less than 50K) and the adversary objective is gender. We also generate synthetic Gaussian datasets to analyze the effect of ally/adversary class overlap in some experiments.

\textbf{Implementation.}
We use fully connected networks (FCNs) for the encoder, allies, and adversaries in the experiments on MIMIC and the synthetic datasets.
The FCN encoder uses ReLU \citep{relu} activation for the hidden layers and tanh activation for the final fully-connected layer, whereas the ally and adversaries use sigmoid activation in the final layer. We use dropout \citep{dropout} and L2-regularization to prevent network overfitting. For FaceScrub, we employ U-Net \citep{unet} for the encoder and Xception-Net \citep{xception} for the ally/adversary as in \cite{bertran2019adversarially}. For Adult, we employ linear FCN as in \cite{advr_rep_learn}. Unless otherwise stated, we set $\alpha = 0.5$ (i.e., equal privacy/predictivity importance). We train to minimize CE loss over 70/30 training/test splits on a system with 8 GB GPU and 64 GB RAM. 

\begin{figure}[t!]
\begin{center}
    \centerline{\includegraphics[width=1.0\columnwidth]{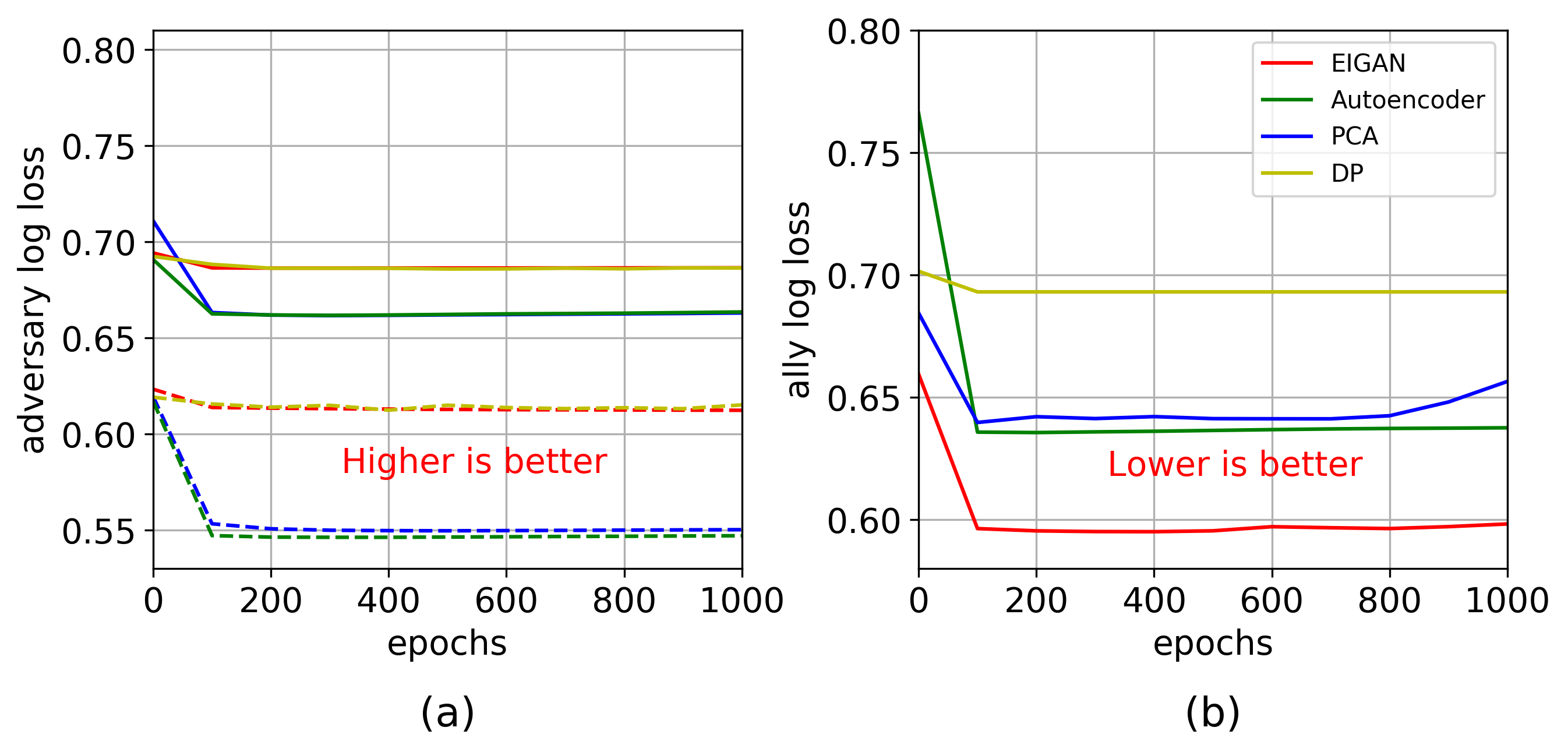}}
    \vskip -0.15in
    \captionof{figure}{\small{Predictivity and privacy comparison between EIGAN and baselines on MIMIC. (a) On the adversary prediction (gender, solid lines and race, dashed lines), EIGAN matches DP's performance (by tuning DP's noise). (b) On the ally prediction (survival), EIGAN achieves noticeable improvement over the baselines.}}
    \label{fig:mimic_half}
\end{center}
\vskip -0.15in
\end{figure}

\begin{figure}[t!]
\begin{center}
\centerline{\includegraphics[width=0.6\columnwidth]{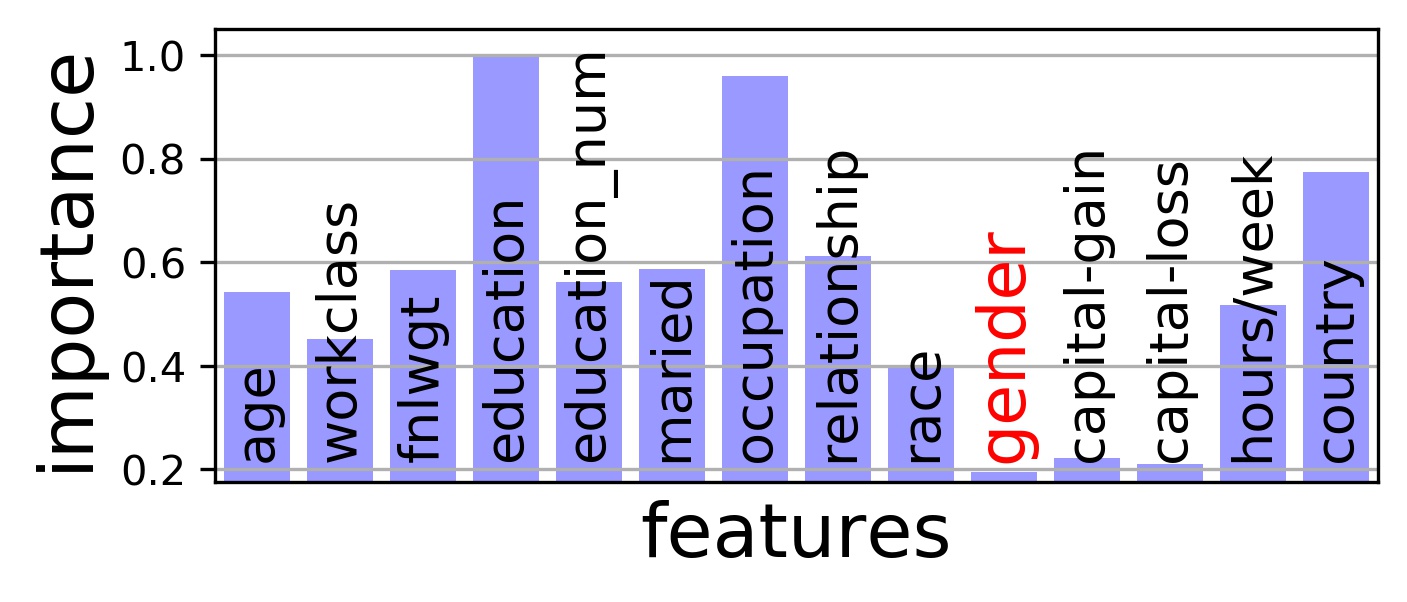}}
    \vskip -0.15in
\caption{\shams{\small{Feature importance derived from EIGAN encoder on Adult dataset results summarized in Table~\ref{table:bertran_compare}.}}}
\label{fig:feature-importance}
\end{center}
\vskip -0.25in
\end{figure}

\begin{table}[t]
    \centering
    \vspace{0.05in}
    \resizebox{\columnwidth}{!}{\begin{tabular}{| c | c | c | }
   
      \hline
      \textbf{Model} & \textbf{Ally} (accuracy) & \textbf{Adv.} (accuracy) \\
      \hline\hline
      \textbf{Resnet152} Unencoded & 0.99 & 0.99\\
       \hline
      Resnet152 & 0.85 & 0.45 \\
      ResNext101 & 0.86 & 0.42 \\
      Resnet101 & 0.88 & 0.64 \\
      Resnet50 & 0.87 & 0.56 \\
      WideResnet101 & 0.85 & 0.42 \\
      VGG19 & 0.77 & 0.42\\
      \hline
    \end{tabular}}
    \vskip -0.05in
    \caption{\small{Accuracy of various architectures used to infer ally (even/ odd) and adversary (digits 0-9) objectives on MNIST encoded using ResNet152-trained EIGAN. We see that the ally accuracies are consistent across network architectures, and the adversary accuracies remain significantly below the performance on the unencoded data.}}
    \label{table:robust_compare}
\end{table}

\textbf{Baselines.} We consider six baselines: principal component analysis (PCA) \citep{pca}, autoencoders \citep{autoencoder}, differential privacy (DP) in the form of Laplace Mechanism as in \cite{privacy_gan}, and the methods in \cite{advr_rep_learn}, \cite{bertran2019adversarially}. Autoencoders and PCA preserve information content and do not have explicit privacy objectives; they are expected to give encoded data that has good predictivity.
PCA chooses the number of components retaining 99\% of the variance, and we train the autoencoder to transform data to the same dimensional space as PCA. As discussed in Section~\ref{sec:intro}, DP is widely used for context-agnostic privacy. For DP, we employ the Laplace mechanism
\citep{privacy_gan}. \cite{bertran2019adversarially} is the most recent state-of-the-art in adversarial PRL; in this case, we use their open-source implementation and compare on the setting described in their paper. We also compare against the closed form optimal solution of \cite{advr_rep_learn} for linear maps on their Adult dataset use case, where \cite{advr_rep_learn} outperforms \cite{advr_info_leakage,louizos2015variational,xie2017controllable,pmlr-v28-zemel13}.

All of our code using PyTorch \citep{NEURIPS2019_9015} and trained models are available at \url{https://github.com/shams-sam/PrivacyGANs}. For each experiments, we report cross-entropy loss and/or accuracy from the testing step of PRL.

\subsection{Centralized EIGAN}
\label{ssec:c-eigan-expt}

\textbf{Performance comparison.} 
We first compare the ally and adversary losses over training epochs between EIGAN, autoencoder, PCA, and DP on the MIMIC dataset in Fig.~\ref{fig:mimic_half}. Note that the recent baselines~\citep{advr_rep_learn,bertran2019adversarially} cannot handle multiple adversary objectives.
It is observed in (a) that EIGAN is able to match the adversary losses of DP, while in (b) the EIGAN ally loss matches that of PCA and autoencoder while outperforming DP by a significant margin. Thus, EIGAN is capable of achieving private representations while simultaneously maintaining the predictivity of the encoded representations.

Next, we compare EIGAN with \cite{advr_rep_learn} and \cite{bertran2019adversarially} on the Adult and Facescrub dataset settings considered in these works, respectively. Note that the linearity requirement in \cite{advr_rep_learn} impedes its usage on non-linear models like the U-Net and Xception-Net employed for Facescrub by \cite{bertran2019adversarially}.  For comparison, we adjust $\alpha$ in \eqref{eqn:log_loss} to equalize the resulting adversary performances between the models. Table~\ref{table:bertran_compare} gives the results: EIGAN matches the performance of \cite{advr_rep_learn}'s optimal closed-form solution on Adult. On the Facescrub dataset, it displays a 47\% improvement in the ally's task of identity recognition when compared to \cite{bertran2019adversarially}. This validates our choice of optimization using cross-entropy loss in \eqref{eqn:log_loss} for PRL over the technique of optimization using KL divergence that is common in recent PRL literature \citep{bertran2019adversarially,gap,ppan}.

\textbf{Robustness of learned representation.} \shams{In Fig.~\ref{fig:feature-importance} we consider the importance placed by EIGAN encoder on input features of Adult dataset for learning the private representations. It can be observed that the importance of gender and it's correlated features is very low. This implies that the learnt representations minimize the signals w.r.t adversary's interest, i.e., gender.} We next consider the robustness of EIGAN's learned representation to ally and adversary architectures that deviate from the one used for training. 
Table~\ref{table:robust_compare} shows the performance of varying architectures (ResNet \citep{he2016deep}, ResNext \citep{resnext}, etc.) for allies and adversaries applied to the data encoded using EIGAN trained with ResNet152 adversary on MNIST. We see that the representations learned by EIGAN are able to obfuscate adversary targets from the other networks. Adversary accuracy remains significantly below the performance on the unencoded data, validating the robustness to differences between simulated and actual adversaries.

\begin{figure}[t!]
\begin{center}
\centerline{\includegraphics[width=0.95\columnwidth]{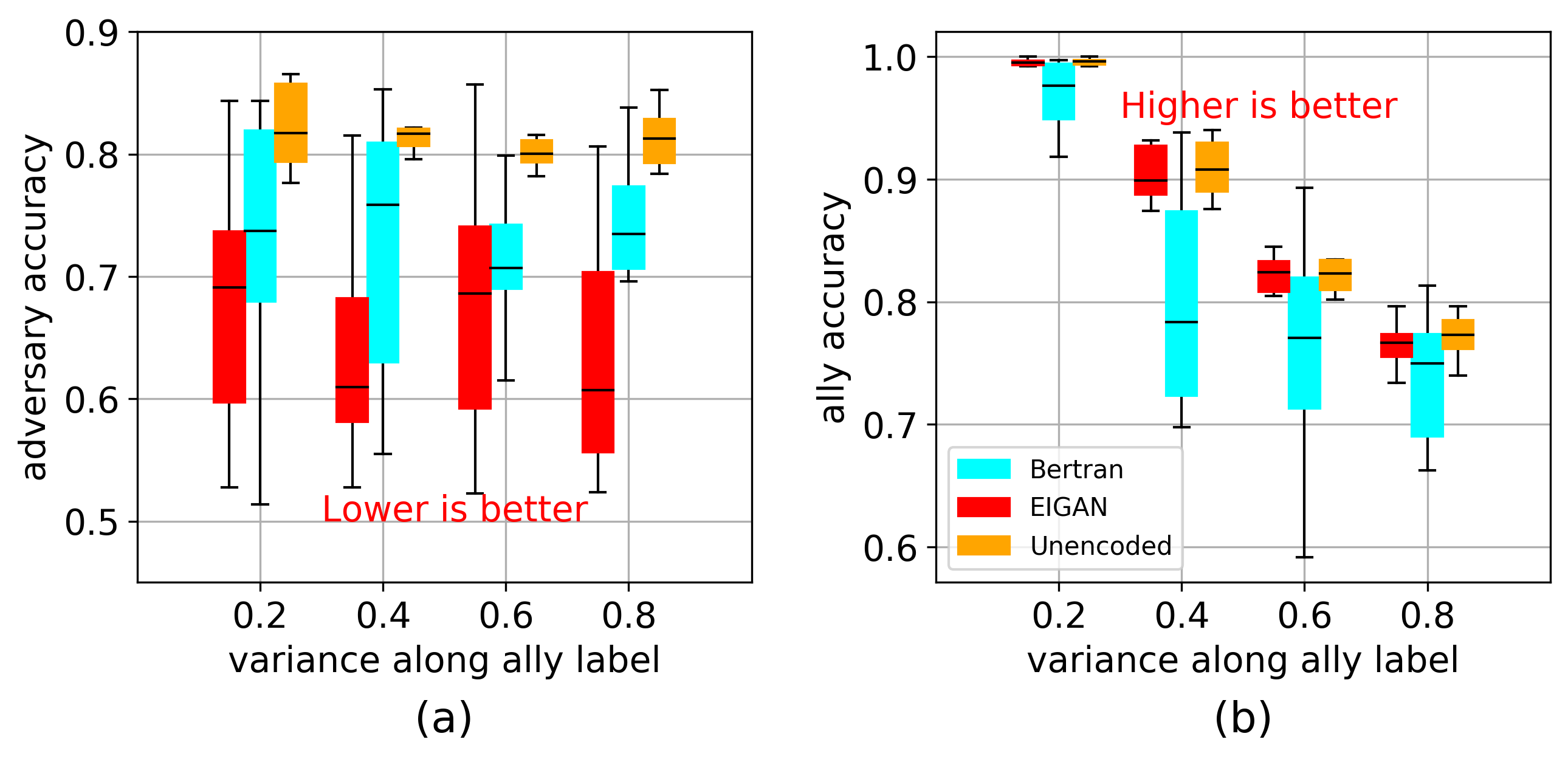}}
    \vskip -0.15in
    \caption{\small{Effect of varying the ally class overlap (by changing the variances of synthetic Gaussian data) on the performance of EIGAN, \cite{bertran2019adversarially}, and the unencoded data. (a) and (b) plot the achieved accuracies of the adv. and ally objectives, respectively. EIGAN is able to consistently outperform both baselines on the adversary objective, and obtains performance close to the unencoded data for the ally.
    }}
    \label{fig:gaussian_bert_ally}
\end{center}
\vskip -0.15in
\end{figure}

\begin{figure}[t!]
\begin{center}
\centerline{\includegraphics[width=1.0\columnwidth]{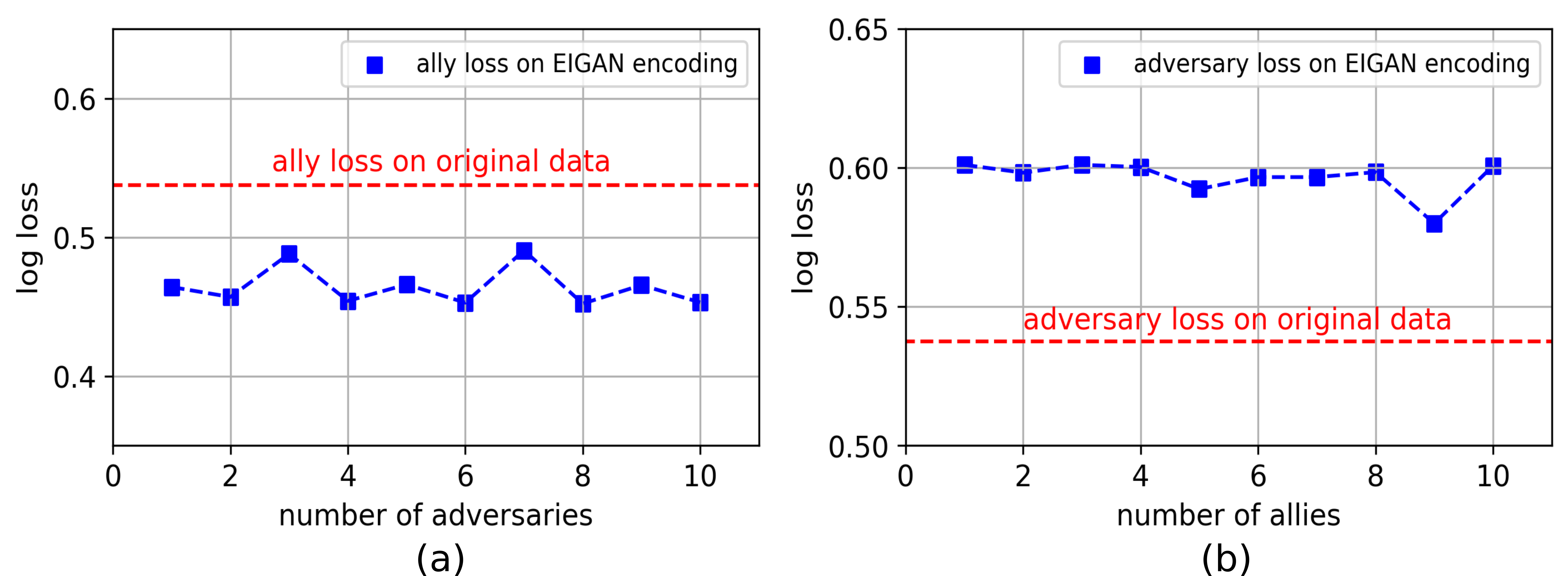}}
    \vskip -0.1in
    \caption{\small{EIGAN's effect of the number of (a) adversaries, and (b) allies on the testing loss for MIMIC-III. The ally/adversary objectives are chosen as different attributes from the source. The achievable loss is reasonably constant and is not affected by addition of more allies/adversaries. }}
    \label{fig:num_ally_advr_comparison}
\end{center}
\vskip -0.25in
\end{figure}

\textbf{Varying ally/adversary overlap.} Next, we consider the effect of class overlap for the ally/adv. objectives on model performance. To do this, we generate a 2D dataset consisting of four Gaussians with means at $(x,y) = (1,1), (1,2), (2,1), (2,2)$, each corresponding to one class. The variance of these Gaussian-distributed classes is adjusted to achieve varying degrees of overlap. Fig.~\ref{fig:distributed_num_nodes}(c) shows an instance of this dataset: the ally is interested in differentiating color, while the adversary wants to differentiate shape. Fig.~\ref{fig:gaussian_bert_ally} shows the effect of the ally label variance on the resulting accuracies for EIGAN, the method in \cite{bertran2019adversarially}, and the unencoded data. 
As the ally variance increases, we observe that (a) the accuracy of the adversary for EIGAN remains consistently lower than that of the others, while (b) the accuracy on the ally objective for EIGAN remains higher than that of \cite{bertran2019adversarially} and is comparable to the unencoded case. Similar results are seen on changing the adversary variance (see Appendix.~\ref{suppsub:eigan-bertran}).

\textbf{Varying system dimensions.} We also consider the impact of the encoding dimension $l$ and the number of allies/adversaries on EIGAN's performance using MIMIC. We summarize our key findings here: (i) We observe (in Fig.~\ref{fig:num_ally_advr_comparison}) that the final test loss obtained by an adversary (ally) under varying number of allies (adversaries) stays reasonably constant. (ii) We find (in Appendix.~\ref{suppsub:eigan-encoder-dim}) that as encoding dimension $l$ is increased, EIGAN converges faster (fewer epochs), and is able to achieve a lower testing loss, even as $l$ exceeds the input dimension.  Thus, encodings are robust to the number of objectives that are included in EIGAN.

\begin{figure}[t!]
\begin{center}
\centerline{\includegraphics[width=1.0\columnwidth]{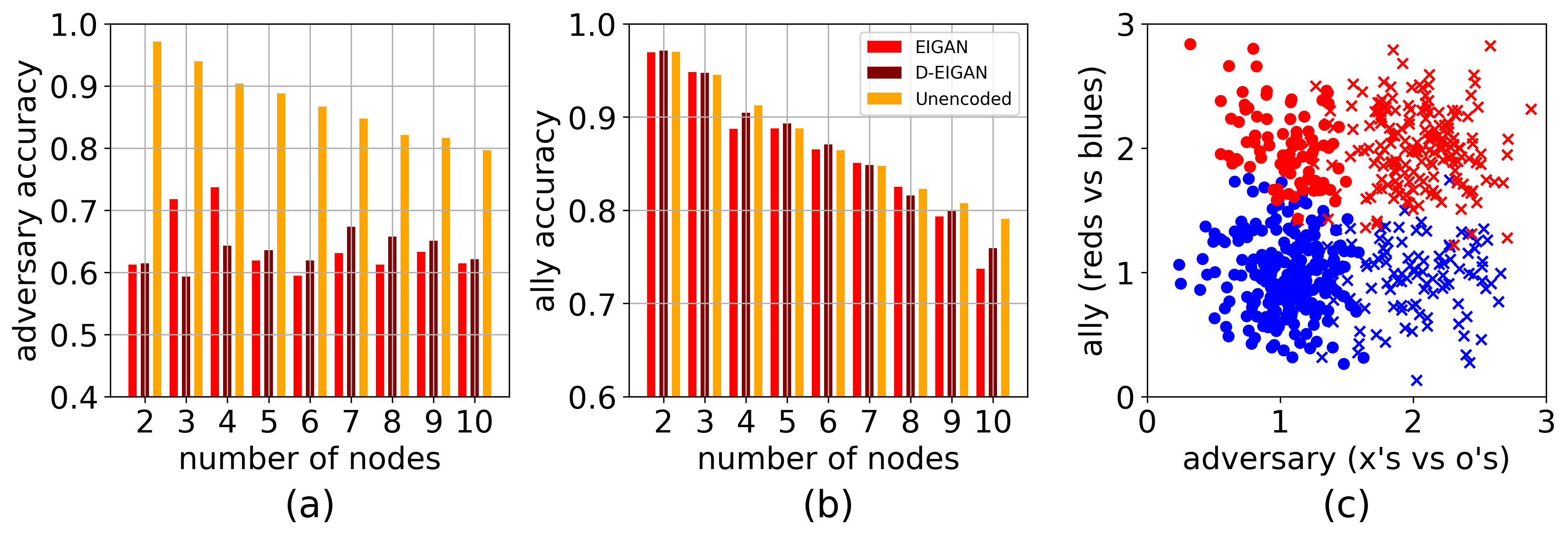}}
    \vskip -0.12in
    \caption{\small{Comparison of (a) adversary and (b) ally performance as the number of nodes in the system is increased from $K = 2$ to $10$, for D-EIGAN ($\phi, \delta = 1$), EIGAN, and unencoded. Node $k$'s data, $k = 1,...,K$ is generated from four Gaussians centered on a unit square, each with $\sigma^2 = 0.1k$, i.e. increasing variance. (c) visualizes the ally (reds vs. blues) and adv. (x's vs. o's) objectives for node $k = 3$. As expected, the ally performs worse with higher $K$, but D-EIGAN is able to match EIGAN's performance.}}
    \label{fig:distributed_num_nodes}
\end{center}
\vskip -0.25in
\end{figure}

\subsection{Distributed EIGAN (D-EIGAN)}
\label{ssec:d-eigan-expt}

\textbf{Varying number of nodes.} For the distributed case, we first study the effect of increasing the number of training nodes $K$. We use synthetic Gaussian data and generate non-i.i.d. data distributions across the nodes by increasing the variance of the Gaussians at each subsequent node $k$ (Fig.~\ref{fig:distributed_num_nodes}(c) shows the distribution for $k=3$). Fig.~\ref{fig:distributed_num_nodes}(a)\&(b) show the resulting ally and adversary accuracies obtained when trained on D-EIGAN, on EIGAN, and on the unencoded data. As $K$ increases, the ally performance degrades in each case, due to the higher variance for each class exhibited in the overall dataset $\mathcal{X}$. Overall, we see that D-EIGAN matches the performance of the centrally-trained EIGAN in both metrics, which shows that distributed learning can yield a comparable solution when all parameters ($\phi = 1$) are synchronized frequently ($\delta = 1$). See Appendix.~\ref{suppsub:d-eigan-numnodes} for results on i.i.d. data.

\textbf{Varying objectives across nodes.} Next, we study the effect of varying ally and adversary objectives across nodes. For this, we consider the MIMIC dataset and allocate the dataset across $K=10$ nodes randomly so that each has a different distribution of patient data. In Fig.~\ref{fig:distributed_eigan_mimic}, we show the accuracies achieved by D-EIGAN on the one ally and two adversary objectives for two cases: (a) when each node has all three objectives, and (b) when each node has the ally objective, but half have one adversary objective and half have the other. The EIGAN performance on the full dataset is included for comparison. The dataset is distributed in a non-i.i.d manner across nodes by non-uniform random sampling. We see that D-EIGAN in (a) only has a slight improvement over (b) in the case of the gender adversary, which indicates that D-EIGAN is robust to varying node objectives, even though the aggregation period has increased ($\delta = 2$) and the fraction of parameters shared has decreased ($\phi = 0.8$) from Fig.~\ref{fig:distributed_num_nodes}. The implication of this is that once a data sample is encoded at a node via D-EIGAN, it can be transferred to another node with different objectives and securely applied to ally tasks there, e.g., referring to the healthcare use case in Section~\ref{sec:intro}, if a patient moves to a different hospital with different health regulations. Similar conclusions are drawn when the data is i.i.d across nodes (see Appendix.~\ref{suppsub:d-eigan-var-obj}).

\begin{figure}[t!]
\begin{center}
\centerline{\includegraphics[width=0.95\columnwidth]{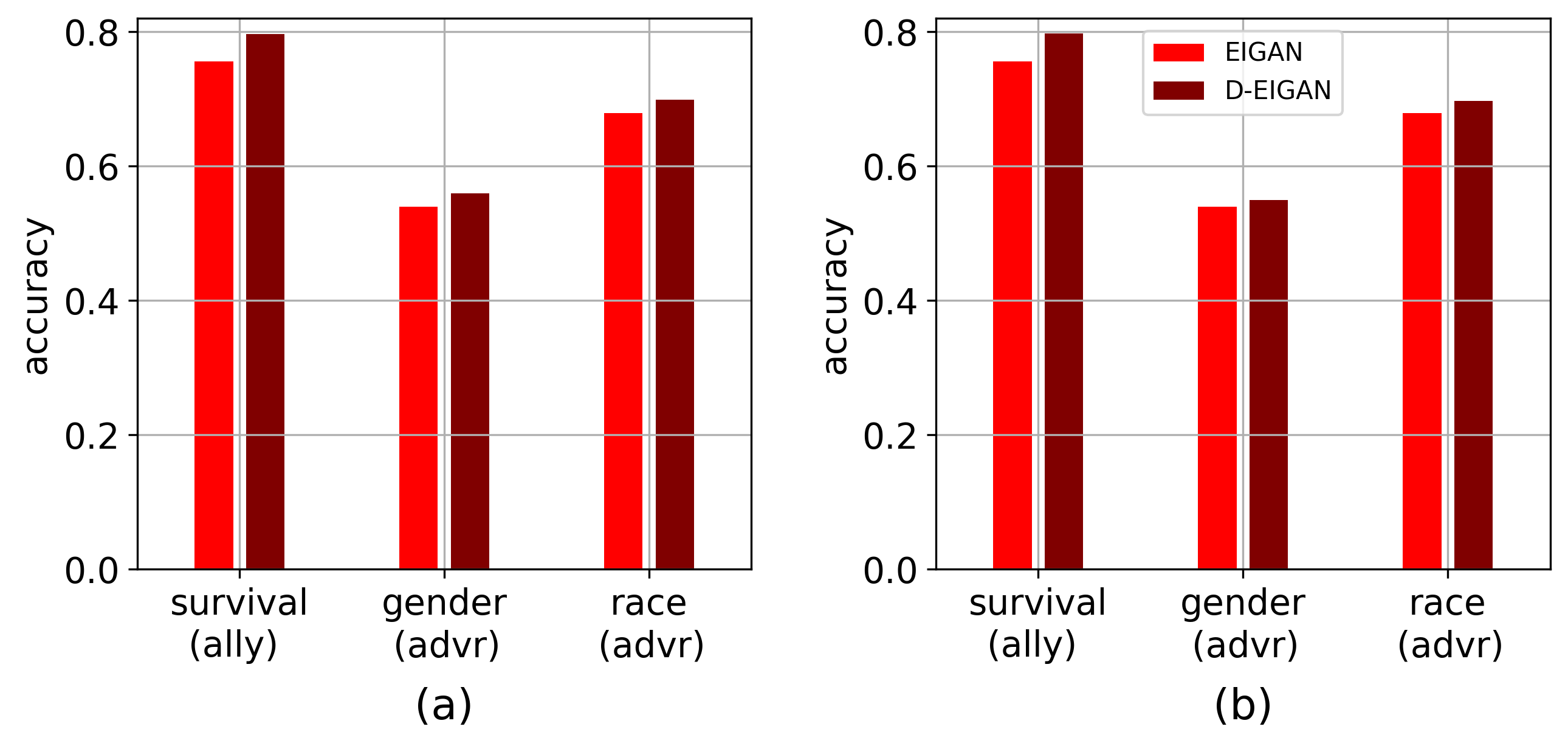}}
    \vskip -0.15in
    \caption{\small{Performance of ally and adversary objectives trained on D-EIGAN ($K=10, \phi=0.8, \delta=2$, non-i.i.d) for MIMIC in the cases of (a) all nodes having all three objectives and (b) each node having the ally but only one of the adversaries. 
    The distribution of objectives across the nodes does not affect the resulting accuracies.}}
    \label{fig:distributed_eigan_mimic}
\end{center}
\vskip -0.15in
\end{figure}

\begin{figure}[t!]
\begin{center}
\centerline{\includegraphics[width=0.95\columnwidth]{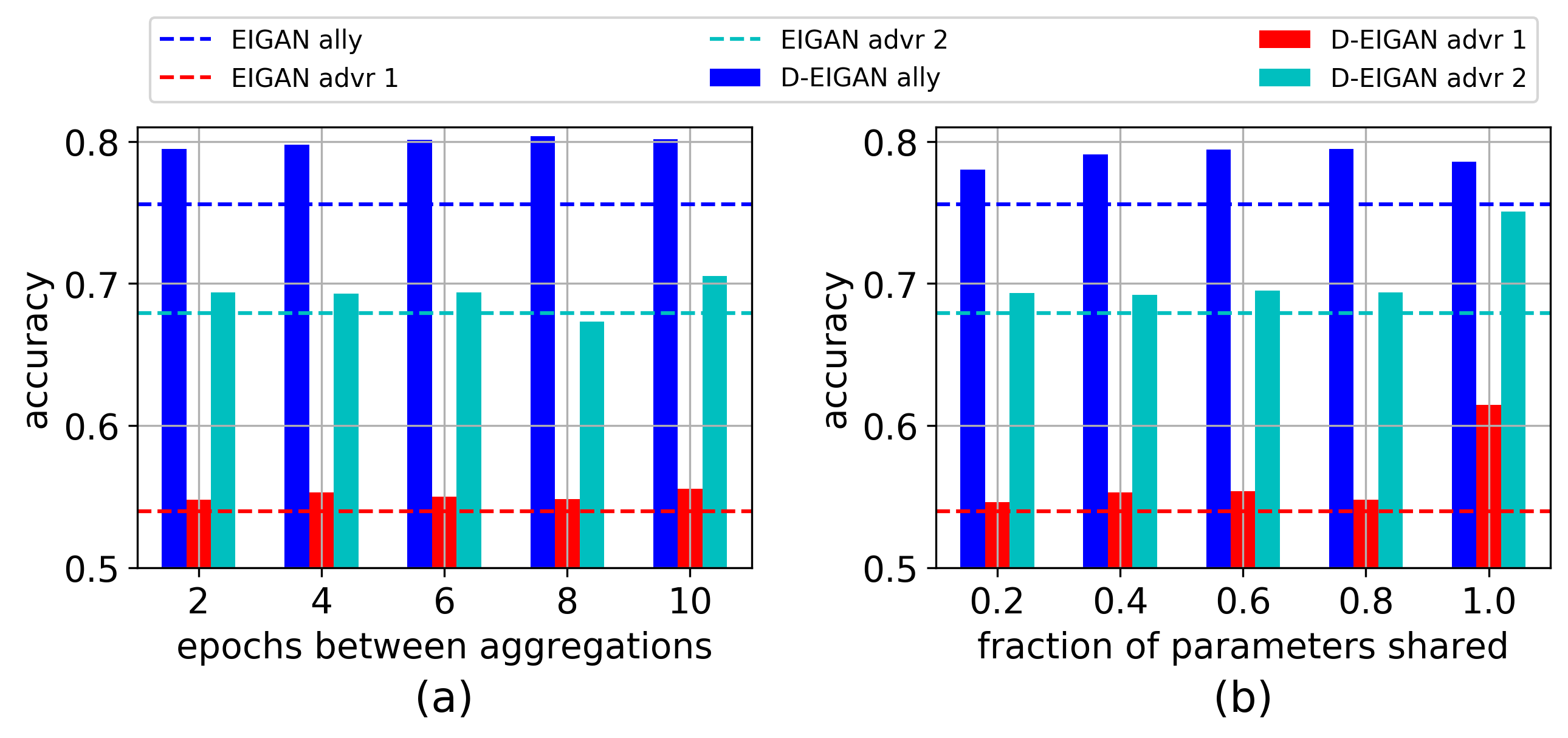}}
    \vskip -0.12in
    \caption{\small{Effect of (a) aggregation frequency $\delta$ ($\smash{\phi = 0.8}$) and (b) sparsification factor $\phi$ ($\smash{\delta = 2}$) on ally/adv. performance on D-EIGAN for the non-i.i.d case in Fig.~\ref{fig:distributed_eigan_mimic}(a). The robust performance shows that D-EIGAN can be applied in communication-constrained environments.}}
    \label{fig:distributed_control_params}
\end{center}
\vskip -0.25in
\end{figure}

\textbf{Varying synchronization parameters.} Finally, we consider the impact of the aggregation period $\delta$ and the sparsification factor $\phi$ on D-EIGAN. This has implications for the communication resources between the nodes and the server required for training, as discussed in Section~\ref{ssec:d-arch}. For this experiment, we use the setting from the experiment in Fig.~\ref{fig:distributed_eigan_mimic}(a), i.e., with non-i.i.d data and all nodes having all three objectives. In Fig.~\ref{fig:distributed_control_params}, we show the performance of D-EIGAN as (a) $\delta$ increases and (b) $\phi$ increases (EIGAN shown for comparison). In (a), we see that D-EIGAN is robust to the number of training epochs between aggregations, implying that it can be increased to limit the frequency of transmissions to/from the server. In (b), we similarly observe generally robust performance as the fraction of sharing changes, though surprisingly, the performance noticeably \textit{decreases} once $\phi$ reaches $1$ and all are shared. A similar effect was observed by \cite{sattler2019robust}, that in the case of distributed model training over non-i.i.d datasets, sparsification actually can \textit{enhance} performance because it minimizes the effect of data bias at each node on the global model. Indeed, in the i.i.d case, we do not observe this effect (see Appendix.~\ref{suppsub:d-eigan-sync-params}). Thus, we conclude that D-EIGAN is well suited for communication-constrained environments.


\section{CONCLUSION}
\label{sec:conc}

We developed the first methodology for generalized and distributable PRL. EIGAN accounts for the presence of multiple allies and adversaries with potentially overlapping objectives, and D-EIGAN addresses privacy concerns and resource constraints in scenarios with decentralized data. We proved that for an optimal encoding, the adversary's output from EIGAN follows a uniform distribution, and that dependencies between ally and adversary interests requires careful balancing of objectives in encoder optimization. Our experiments showed that EIGAN outperforms six baselines in jointly optimizing predictivity and privacy on different datasets and system settings. They also showed that D-EIGAN achieves comparable performance to EIGAN with different numbers of training nodes and as the training parameters vary to account for communication constraints.


\subsubsection*{Acknowledgements}
\vspace{-1mm}
S.S. Azam, C. Brinton, and S. Bagchi were supported by the Northrop Grumman Cybersecurity Research Consortium.


{
\bibliographystyle{plainnat}
\bibliography{ms}
}


\clearpage
\appendix

\thispagestyle{empty}
\onecolumn \makesupplementtitle

\setcounter{section}{0}
\renewcommand{\thesection}{\Alph{section}}
\renewcommand{\thefigure}{\arabic{figure}}


\section{Propositions and Proofs}
\label{supp:proofs}

\subsection{Proof of Proposition~\ref{prop:uniform}}

Suppose $\smash{\hat{Y}_{A_i} = p_{_{A_i}}\left({Y}|E(\mathcal{X})\right)}$ and $\smash{\hat{Y}_{V_j} = p_{_{V_j}}\left({Y}|E(\mathcal{X})\right)}$, where $p_{_{A_i}}\left({Y}|E(\mathcal{X})\right)$ and $p_{_{V_j}}\left({Y}|E(\mathcal{X})\right)$ denote the posterior probabilities of successful inference of target labels $Y \sim \mathcal{Y}_{A_i}$ and sensitive labels $Y \sim \mathcal{Y}_{V_j}$ for ally $A_i$ and adversary $V_j$, respectively, given the outputs encoder $E$ provides for the dataset $\mathcal{X}$. 
Then, the utilities in  \eqref{eq:lossFunctions} can be expressed as
\vspace{-4mm}

\begin{equation}\label{eqn:restated_utility}
    \begin{aligned}
  u_{_{A_i}} = \mathbb{E}_{{Y}\sim \mathcal{Y}_{_{A_{i}}}} \left[ \log\,\hat{Y}_{A_i} \right]; u_{_{V_j}} = \mathbb{E}_{Y\sim \mathcal{Y}_{_{{V}_j}}} \left[ \log\,\hat{Y}_{V_j} \right],
    \end{aligned}
\end{equation}
\vspace{-3mm}

\noindent where $\smash{1\leq i\leq n}$ and $\smash{1\leq j\leq m}$. Let $H_{Q} = \mathbb{H}(P, Q)$ denote the cross-entropy of $Q$ with respect to $P$ defined as $H_Q = \mathbb{H}(P, Q) = \mathbb{E}_{x\sim P}\left[ -\log Q \right]$, then \eqref{eqn:restated_utility} can be re-stated as:
\vspace{-4mm}

\begin{equation}\label{eqn:entropy}
    \begin{aligned}
  & u_{_{A_i}} = -H_{A_i} = - \mathbb{H}(Y\sim\mathcal{Y}_{A_i}, \hat{Y}_{A_i}),~~~~1\leq i\leq n, \\
   &u_{_{V_j}} = -H_{V_j} = - \mathbb{H}(Y\sim\mathcal{Y}_{V_j}, \hat{Y}_{V_j}),~~~~1\leq j\leq m.
    \end{aligned}
\end{equation}
\vspace{-3mm}

\noindent The maximization of ally utilities $u_{A_i}$ and minimization of adversary utilities $u_{V_j}$ $\forall i,j$ in the optimization objective \eqref{eqn:minimax_game3} can be re-written as minimization of its negative given by,
\vspace{-4mm}

{
\begin{equation}
\begin{aligned}
U' = -\sum_{i=1}^{n} \alpha_{A_{i}} u_{A_{i}} + \sum_{j=1}^{m} \alpha_{V_j} u_{V_{j}} = \sum_{i=1}^{n} \alpha_{A_{i}} H_{A_{i}} - \sum_{j=1}^{m} \alpha_{V_j} H_{V_{j}}.
\end{aligned}
    \label{eqn:minimax_game_entropy_last}
\end{equation}
}
\vspace{-4mm}

Through \eqref{eqn:minimax_game_entropy_last}, it can be observed that the minimization occurs when entropy of allies $\sum_{i=1}^{n} \alpha_{A_{i}} H_{A_{i}}$ is minimized while that of adversaries $\sum_{j=1}^{m} \alpha_{V_j} H_{V_{j}}$ is maximized. Using the definition of entropy, each of the allies and adversaries has a global optimum and can be optimized separately if their labels are non-overlapping. Note that ally and adversary entropies are non-negative, and given a fixed encoder $E$, the sum of ally entropies is minimized when individual entropies are minimized.
For each ally, individual entropy $H_{A_i}$ is minimized when $\hat{Y}_{A_i}$ takes the value of $1$ $\forall i$ as every ally label is then predicted correctly. Similarly for adversaries, each individual entropy $H_{V_j}$ is maximized when $ \hat{Y}_{V_j} = 1 / |Y_{V_j}|$ is the uniform distribution. Thus, it can be seen that, at the optimal solution, the adversaries' output follows a uniform distribution, as it minimizes the overall entropy in \eqref{eqn:minimax_game_entropy_last}, or equivalently maximizes the utility in \eqref{eqn:minimax_game3}.

Given that $(A_i,V_j) \in \mathcal{O}$ is the set of all $(i,j)$ pairs of allies $A_i$ and adversaries $V_j$ for which $Y_{A_i} \cap Y_{V_j} \neq \emptyset$, the ally and adversary objectives in  \eqref{eqn:minimax_game_entropy_last} are overlapping if $\mathcal{O} \neq \emptyset$. Given that the encoder is fixed, for allies/adversaries not included in $\mathcal{O}$, the associated utilities can be independently optimized. We are thus left with the maximization of the following:
\vspace{-3mm}

\begin{equation}
    U_{\mathcal{O}} = \sum_{(A_i,V_j) \in \mathcal{O}}  \alpha_{A_i} \cdot u_{A_i} - \alpha_{V_j} \cdot u_{V_j}.
    \label{eqn:pos_corr}
\end{equation}
\vspace{-2.5mm}

\noindent For the $k$th element in $\mathcal{O}$, $(A_{i(k)}, V_{j(k)})$, we have $Y_{A_{i(k)}}(c) = \smash{Y_{V_{j(k)}}(c) \; \forall c \in \mathcal{C}_k\;  \forall k;}$, where $\mathcal{C}_k$ is the set of indices of elements in $Y_{A_{i(k)}} \cap Y_{V_{j(k)}} \neq \emptyset$. Separating the indices $c$ for which the ally/adversary try to predict the same label (i.e. $\smash{u_{A_i}(c) = u_{V_j}(c)}$), we can express \eqref{eqn:pos_corr} as follows:
\vspace{-3mm}

{
\begin{equation}
\hspace{-4mm} 
\begin{aligned}
        U_{\mathcal{O}} = \hspace{-3mm} \quad \sum_{k} \Biggl( \underbrace{\sum_{c\in\mathcal{C}_k} (\alpha_{A_i} - \alpha_{V_j})  u_{A_i}(c)}_{\text{utility w.r.t. overlapping labels, }U_{\mathcal{O}^+}} +  \underbrace{\sum_{c\notin\mathcal{C}_k} \alpha_{A_i} u_{A_i}(c) - \alpha_{V_j} u_{V_j}(c)}_{\text{utility w.r.t. non-overlapping labels}, U_{\mathcal{O}^-}}  \Biggl).
\end{aligned}
    \label{eqn:pos_corr2}
    \hspace{-4mm} 
\end{equation}}
\vspace{-2.5mm}

\noindent The utilities in \eqref{eqn:pos_corr2} reward only one of the two discriminators $(A_i, V_j) \in \mathcal{O}$ predicting on overlapping label $c \in \mathcal{C}$ if $\alpha_{A_i} \neq \alpha_{V_j}$. If $\alpha_{A_i} = \alpha_{V_j}$ for $(A_i, V_j) \in \mathcal{O}$, then $U_{\mathcal{O}^+} = 0$, and no  optimization occurs w.r.t. the overlapping labels in $Y \sim \mathcal{Y}_{A_i}$.

\subsection{Proposition 2}
\label{suppsub:proof-2}
\begin{prop}\label{prop:nonuniform}
 Assume that the number of labels of interest is the same among all the allies and adversaries. For any adversary $V_j$, the distribution of its prediction over its set of labels of interest
does not follow a uniform distribution if sufficient weight is given to the ally utilities (i.e., $\alpha_{A_i}$, $\forall A_i$, is sufficiently large) and the distribution of prediction of one ally $A_i$, can be defined as a linear combination of the distribution of predictions of $V_j$ and that of other allies/adversaries.
\end{prop}

\begin{proof}
Without loss of generality, consider a system with one ally network with a scalar output $\hat{Y}_{A}$ and $m$ adversary networks with scalar outputs $\hat{Y}_{{V_j}}$ for $1 \leq j \leq m$. The true distribution of each predicted output is $\mathcal{Y}_{A}$ for the ally and $\mathcal{Y}_{{V_j}}$ for the adversaries, and $Y_{A}$ and $Y_{{V_j}}$ are the actual labels drawn from those distributions respectively. The true values and predictions between that of the ally and the adversaries have the relation, $Y_{A} = \sum_{j=1}^{m} w_{j} Y_{{V_j}}$, and $\smash{\hat{Y}_{A} = \sum_{j=1}^{m} w_{j} \hat{Y}_{{V_j}}}$ where $w_{j}$ is scaling weight.  The cross entropy of the entire system is given by $U = \alpha_A Y_{A} \log(\hat{Y}_{A}) - \sum_{j=1}^{m} \alpha_{V_j} Y_{{V_j}} \log(\hat{Y}_{{V_j}})$. Optimizing for the output of a specific adversary $V_n$, we obtain:
\vspace{-2mm}

{
\begin{equation}
    \hat{Y}_{{V_n}} =  \frac{\sum_{j\neq n} w_{j} \hat{Y}_{{V_j}} } {\alpha_A Y_{A} w_{n}}  \left( \frac{1}{\alpha_{V_n} Y_{{V_n}}} - \frac{1}{\alpha_A Y_{A}} \right)^{-1}.
\label{eq:prop3_ent}
\end{equation}
}
\vspace{-3mm}

\noindent Notably, $\hat{Y}_{{V_n}}$ only returns a non-uniform distribution when $\alpha_{V_n} Y_{{V_n}} < \alpha_A Y_{A}$. If the weight $\alpha_A$ is not large enough to maintain the inequality, the value of $\hat{Y}_{{V_n}}$ cannot be obtained via \eqref{eq:prop3_ent} and will have a uniform distribution.
If $\alpha_{V_n} Y_{{V_n}} = \alpha_A Y_{A}$, then the cross entropy $U=0$ and no optimization occurs. 
\end{proof}

\subsection{Proof of Proposition~\ref{prop:distributed_iid}}
\label{suppsub:proof-3}
Given that the global encoder is the average of the local encoders in the federated learning procedure for a single synchronization across $K$ local nodes, the maximization of the expectation in \eqref{eqn:minimax_game4} can be described as the maximization of ally utilities and minimization of adversary utilities given by:
\vspace{-4mm}

{
\begin{equation}
\begin{aligned}
U  = \frac{1}{K} \sum_{k=1}^{K} \left( \sum_{i=1}^{n_{(k)}} \alpha_{A_i^{(k)}} u_{A_i^{(k)}} - \sum_{j=1}^{m_{(k)}} \alpha_{V_j^{(k)}} u_{V_j^{(k)}}  \right).
\end{aligned}
    \label{eqn:distributed_entropy}
\end{equation}
}
\vspace{-2mm}

In \eqref{eqn:distributed_entropy}, $A_{i}^{(k)}$ and $V_{i}^{(k)}$ refer to the $i^{\text{th}}$ ally or adversary of the $k^{\text{th}}$ local node. Since data at each node is i.i.d, the distributions $\mathcal{Y}$ are the same at each node, and thus each node has the same objective function. Using the result of Prop.~\ref{prop:uniform} and assuming that $A_{i}^{(k_1)}, V_{j}^{(k_1)} = A_{i}^{(k_2)}, V_{j}^{(k_2)}$ $\forall i,j, k_1, k_2$ (i.e., the ally and adversary labels are same across all nodes), the output of the adversaries at each node follow a uniform distribution. 

The ally and adversary objectives in \eqref{eqn:distributed_entropy} are overlapping if $\mathcal{O} \neq \emptyset$ given that $(A_i,V_j) \in \mathcal{O}$ is the set of all $A_i, V_j$ pairs for which $Y_{A_i} = Y_{V_j}$. Since each of the local nodes have the same overlapping ally/adversary labels with potentially different weights $\alpha_{A_{i}^{(k)}}$ and $\alpha_{V_{j}^{(k)}}$, their utilities can be expressed using entropy as in \eqref{eqn:entropy}. The final optimization of the distributed system can be expressed as the minimization of following:

{
\begin{equation}
    U_{\mathcal{O}} = \quad \sum_{(A_i,V_j) \in \mathcal{O}}  \left(\sum_{k=1}^{K}(\alpha_{A_i^{(k)}} - \alpha_{V_j^{(k)}}) \cdot u_{A_i^{k}}\right).
    \label{eqn:dist_pos_corr}
\end{equation}
}
\vspace{-3mm}

The entropy values given in \eqref{eqn:dist_pos_corr} reward only one of the two discriminators predicting label $Y_{A_i}$ if $\sum_{k=1}^{K}\alpha_{A_i^{(k)}} \neq \sum_{k=1}^{k}\alpha_{V_i^{(k)}}$. If $\sum_{k=1}^{K}\alpha_{A_i^{(k)}} = \sum_{k=1}^{K}\alpha_{V_i^{(k)}}$, these two networks have no contribution to $U_{\mathcal{O}}$, and no optimization occurs..

 \vspace{-.5mm}
\subsection{Proposition 4}
\label{suppsub:proof-4}
\vspace{-1mm}
\begin{manualtheorem}{4}
\label{prop:dist_obj}
If the allies and adversaries located at the $K$ nodes of D-EIGAN have non-overlapping target sets, i.e.,
${Y}_{A^{(k)}} \neq {Y}_{A^{(k')}}$ and ${Y}_{V^{(k)}} \neq {Y}_{V^{(k')}}$, $1 \leq k, k' \leq K$, then individual encoders under D-EIGAN 
consider the union of these local allies, $\bigcup_{k=1}^{K} {Y}_{A^{(k)}}$, and adversaries ,$\bigcup_{k=1}^{K} {Y}_{V^{(k)}}$ for optimization as a result of the global aggregation step. The weights $\alpha_{A_i^{(k)}}\text{ and } \alpha_{V_i^{(k)}}$ associated with the allies/adversaries are scaled by the ratio of the number of nodes that implement them locally to the total number of nodes.
\end{manualtheorem}
 
\begin{proof}
Without loss of generality, consider a two network D-EIGAN. Let node 1 have 2 allies and 1 adversary with objectives: $Y_{A_c}$, $Y_{A_1}$, and $Y_{V_1}$, and node 2 have 2 allies and 1 adversary with objectives: $Y_{A_c}$, $Y_{A_2}$ and $Y_{V_2}$. Here, objective $Y_{A_c}$ is common among them, while the rest are different. Utilities of individual nodes can be calculated using \eqref{eqn:minimax_game3}:
\vspace{-4mm}

{
\begin{align}
    U^{(1)} &= \alpha_{A_c}\cdot u_{A_c} + \alpha_{A_1}\cdot u_{A_1} - \alpha_{V_1}\cdot u_{V_1},\\
    U^{(2)} &= \alpha_{A_c}\cdot u_{A_c} + \alpha_{A_2}\cdot u_{A_2} - \alpha_{V_2}\cdot u_{V_2}.
\end{align}
}
\vspace{-5mm}

\noindent Under federated training, the equivalent loss function that is optimized by the D-EIGAN can be calculated using \eqref{eqn:minimax_game4}:
\vspace{-2mm}

{
\begin{equation}
    U = \alpha_{A_c}\cdot u_{A_c} + \frac{\alpha_{A_1}}{2}\cdot u_{A_1} - \frac{\alpha_{V_1}}{2}\cdot u_{V_1} + \frac{\alpha_{A_2}}{2}\cdot u_{A_2} - \frac{\alpha_{V_2}}{2}\cdot u_{V_2},
\end{equation}
}
\vspace{-3mm}

\noindent which shows that the overall objective under D-EIGAN considers all the objectives, but the associated weights are lower for non-common allies/adversaries. In contrast to a D-EIGAN where all allies and adversaries are common across nodes, the  difference is the weights associated with objectives.
\end{proof}


\section{Pseudocode of EIGAN}
\label{supp:algorithms}

In this section, Algorithm~\ref{alg:eigan} presents the step-by-step implementation of centralized EIGAN training. It is an iterative mini-batch stochastic gradient descent procedure in which we update the weights for the encoder and allies/adversaries alternately until convergence. The learning rates for the encoder, allies and adversaries are controlled using parameters $\eta_{E}$, $\eta_{A}$, and $\eta_{V}$.

\renewcommand{\thealgorithm}{A1}
\begin{algorithm}[h!] {\small 
   \caption{EIGAN training}
   {\footnotesize
\label{alg:eigan}
\begin{algorithmic}[1]
   \State \textbf{\underline{Notations:}}
   \vspace{0.05in}
    \State ${(\cdot)_j}$ denotes the value for the $j$th minibatch
    \State ${\mathcal{L}_{A_{i}}}$ denotes the loss of ally $A_{i}$
    \State ${\mathcal{L}_{{V_{i}}}}$ denotes the loss of the adversary $V_{i}$
      \State $\eta_{_E},\eta_{_A}$, $\eta_{_V}$: learning rates of the encoders, allies and adversaries
   \vskip 0.05in
   \State \textbf{\underline{Training:}}
  \vspace{0.05in}
  \State initialize $\alpha$ used in loss function \eqref{eqn:log_loss}
  \State initialize $\theta_{A_i}\text{'s}$ and $\theta_{V_j}\text{'s}$ and $\theta_E$ to start the training
   \For{number of training epochs}
        \State Sample a minibatch set $J$ of data points
        \State Compute encoder loss using \eqref{eqn:log_loss}: ${\mathcal{L}}_{E} = \frac{1}{|J|} \sum_{j \in J} \left(\mathcal{L}_{{E}}\right)_j$
        \State Update encoder parameters: $\theta_{E} \leftarrow \theta_{E} - {\eta}_{_E} \cdot \nabla_{\theta_{E}} {\mathcal{L}}_{E}$
        \State Compute allies/adversaries losses using \eqref{eqn:log_loss}:
        \begin{center}
            ${\mathcal{L}_{A_{i}}} = - \frac{1}{|J|} \sum_{j \in J} \left({\mathcal{L}_{A_{i}}}\right)_{j}$, \;  ${\mathcal{L}_{{V_{i}}}} = - \frac{1}{|J|} \sum_{j \in J} \left({\mathcal{L}_{{V_{i}}}}\right)_{j}$
        \end{center} 
        \State Update local allies/adversaries parameters: \begin{center}$\theta_{A_{i}}\leftarrow \theta_{A_{i}} - \eta_{_A} \cdot \nabla_{\theta_{A_{i}}} {\mathcal{L}}_{A_{i}}$, \; $\theta_{V_{i}} \leftarrow \theta_{V_{i}} - \eta_{_V} \cdot \nabla_{\theta_{V_{i}}} {\mathcal{L}}_{V_{i}}$ \end{center} 
   \EndFor
\end{algorithmic}
}
}
\end{algorithm}

\clearpage
\section{Additional Proof of Concept Visualizations}
\label{supp:poc-vis}
In this section, we include additional proof of concept visualizations beyond those presented in Fig.~\ref{fig:octant_advr} from Sec.~\ref{ssec:c-arch}.

The first experiment uses a synthetic dataset comprising 4 sets of Gaussian distributed points in 2-D around the means (-0.5, -0.5), (-0.5, 1.5), (1.5, -1.5) and (1.5, 1.5) as shown in Fig.~\ref{fig:quad}(a). We implement EIGAN with the ally objective to distinguish between reds and blues and adversary objective to segregate x's and o's. This is the simplest case we consider, as there is a single ally and single adversary, each with binary labels. Decision boundaries are linear. We thus use a logistic regression classifier as it has a convex loss function. The encoder is a neural network with a single hidden layer and output dimension $l = 2$. The learnt representation in Fig.~\ref{fig:quad}(b) is intuitive: it maintains linear separability among ally classes, i.e., reds vs blues, but ensures a collapse of adversary classes. 

\begin{figure}[h!]
\begin{center}
\centerline{\includegraphics[width=0.6\columnwidth]{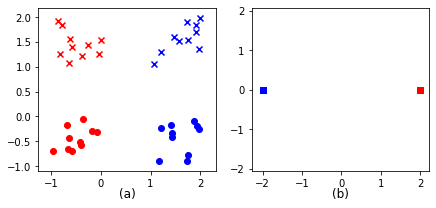}}
\vskip -0.1in
\caption{\small{(a) Quadrant dataset with four groups of points, one ally, and one adversary. The points are linearly separable with regard to the ally's (classifying reds/blues) and an adversary's (classifying x's/o's) objectives. (b) EIGAN learns a representation that collapses the axes along the adversary's objective  while enhancing separation along the ally's. 
}}
\label{fig:quad}
\end{center}
\vskip -0.2in
\end{figure}

Next we consider a dataset with non-linear decision boundary as shown in Fig.~\ref{fig:circle} (a). The ally is interested in a decision boundary between the red and the blue circle, while the adversary is interested in the upper vs. lower semicircle, i.e., x's vs o's. The same encoder is used as in the previous experiment. We use a neural network with a single hidden layer as the ally and adversary because the ally's decision boundary is not linearly separable. Fig.~\ref{fig:circle}(b) shows the learnt representation, which achieves a separability in the encoded space that is qualitatively similar to the representation learnt in Fig.~\ref{fig:quad}(b).

\begin{figure}[h!]
\begin{center}
\vskip -0.1in
\centerline{\includegraphics[width=0.6\columnwidth]{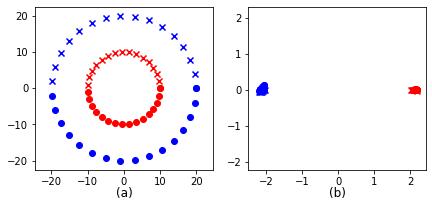}}
\vskip -0.1in
\caption{\small{(a) Circle dataset with the same objectives as Figure \ref{fig:quad} but ally classes (reds vs blues) are not linearly separable. (b) EIGAN learns a similar transformation, making the ally's classification task linearly separable.}}
\label{fig:circle}
\end{center}
\vskip -0.2in
\end{figure}

\begin{figure}[h!]
\begin{center}
\vskip -0.2in
\centerline{\includegraphics[width=0.7\columnwidth]{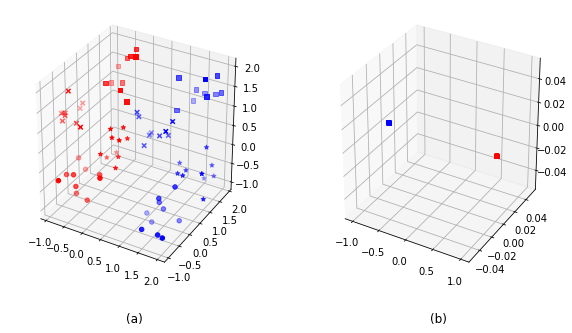}}
\vskip -0.1in
\caption{\small{(a) Octant dataset with eight groups of points, one ally, and two adversaries. The ally is interested in classifying reds/blues while the adversaries are interested in separation along other axes. (b) EIGAN collapses the two adversary dimensions while maintaining separability for the ally.}}
\label{fig:octant}
\end{center}
\vskip -0.2in
\end{figure}

For the final experiment, we extend EIGAN from single ally and single adversary to multiple allies and adversaries. We consider two different cases: EIGAN with (i) 2 ally and 1 adversary objective, and (ii) 1 ally and 2 adversaries presented in Fig.~\ref{fig:octant}. Case (i) is what was presented in Fig.~\ref{fig:octant_advr} of Sec.~\ref{ssec:c-arch}. In the case (ii), we have 8 set of Gaussian distributed points, one in each octant as shown in Fig.~\ref{fig:octant}(a). The ally wants to separate reds vs blues, and the adversaries want to separate along the other axes, i.e., top vs bottom (adversary 1) and squares and stars vs x's and o's (adversary 2). The learnt representation only preserves ally's dimension of variation, i.e. reds vs blues. All the other dimensions are collapsed.

\section{Additional EIGAN Experiments}
\label{supp:eigan-expt}

\begin{figure*}[h!]
\begin{center}
    \centerline{\includegraphics[width=0.9\textwidth]{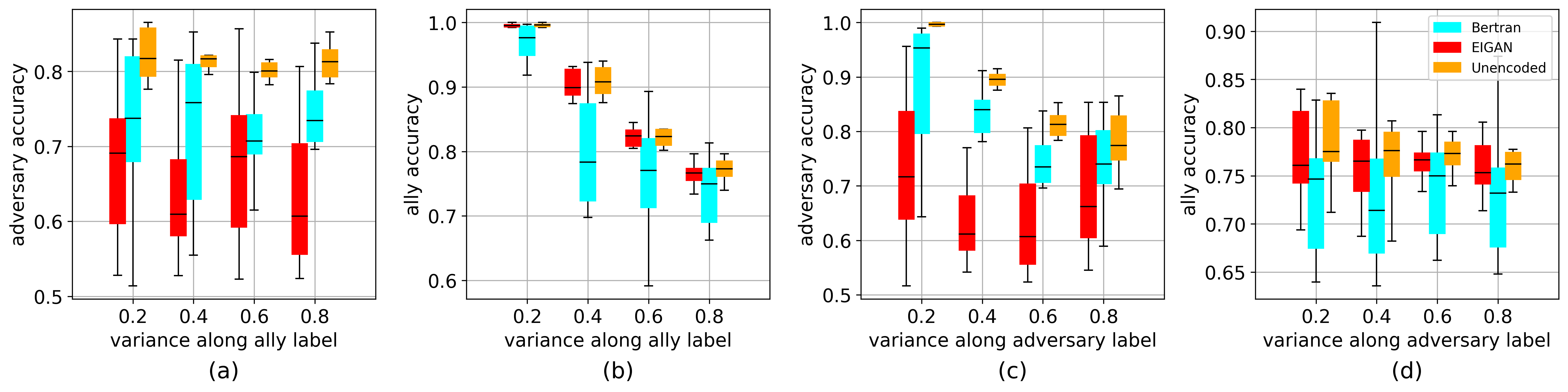}}
\vskip -0.1in
\caption{\small{Effect of change in ally overlap (a-b) and adversary overlap (c-d) on the performance of EIGAN, and its comparison with the unencoded data as well as the method in Bertran et al. \cite{bertran2019adversarially}. EIGAN is able to consistently outperform both baselines on the adversary objective, and obtains performance close to the unencoded data for the ally. 
}}
\label{fig:gaussian_bertran}
\end{center}
\vskip -0.2in
\end{figure*}

\subsection{Comparison with the method in \texorpdfstring{\cite{bertran2019adversarially}}{(Bertran et al. 2019)}}
\label{suppsub:eigan-bertran}

Fig.~\ref{fig:gaussian_bert_ally} from Sec~\ref{ssec:c-eigan-expt} presented a comparison of \cite{bertran2019adversarially} with EIGAN on synthetic Gaussian data. Fig.~\ref{fig:gaussian_bertran} is the extended version, presenting additionally the comparison of (c) adversary and (d) ally performance as we alter the class overlap between adversary labels. Consistent with the conclusions presented in Sec~\ref{ssec:c-eigan-expt} for the ally variation, EIGAN outperforms Bertran et al. consistently  as the adversary exhibits more variance. The $p$-values of the improvements EIGAN makes over the method in \cite{bertran2019adversarially} are below $0.002$ in all 16 cases of comparisons between boxplot distributions.

\label{suppsub:eigan-mimic}
\begin{figure*}[h!]
\begin{center}
\centerline{\includegraphics[width=0.85\textwidth]{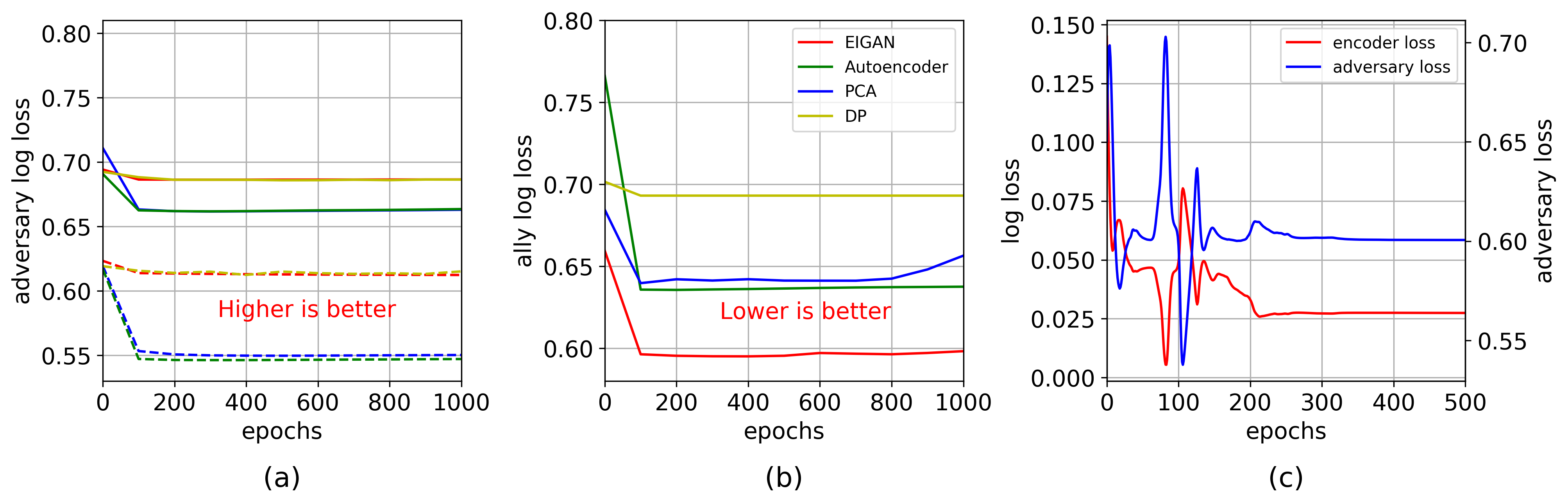}}
\vskip -0.1in
\caption{\small{Predictivity and privacy comparison between EIGAN and the baselines across one ally and two adversaries on the MIMIC-III dataset. (a) On the adversary objectives (gender prediction, solid lines and race prediction, dashed lines) EIGAN matches DP's performance (by design of the experiment, as determined by the selection of the DP $\epsilon$ parameter). Hence, the red and the khaki colored curves overlap. (b) On the ally objective (survival prediction), EIGAN achieves noticeable improvement over the baselines. (c) EIGAN training converges after initial oscillations corresponding to the minimax game.}}
\label{fig:mimic_comparison}
\end{center}
\vskip -0.3in
\end{figure*}

\subsection{Comparison on MIMIC-III}

Here we present the extended results of Fig.~\ref{fig:mimic_half} from Sec.~\ref{ssec:c-eigan-expt}  which compared EIGAN against baselines on the MIMIC dataset. Fig.~\ref{fig:mimic_comparison}(c) shows the loss progression of encoder and adversary as the EIGAN training proceeds. It can be observed that increase/decrease in encoder loss is corresponding to the decrease/increase in adversary loss during the same epoch, consistent with the definition of the encoder loss in \eqref{eqn:log_loss}. The magnitude of the oscillations decreases as we progress through the training and eventually the networks (i.e., the players in the game) reach a steady state.

\begin{figure*}[h!]
\begin{center}
\centerline{\includegraphics[width=0.85\textwidth]{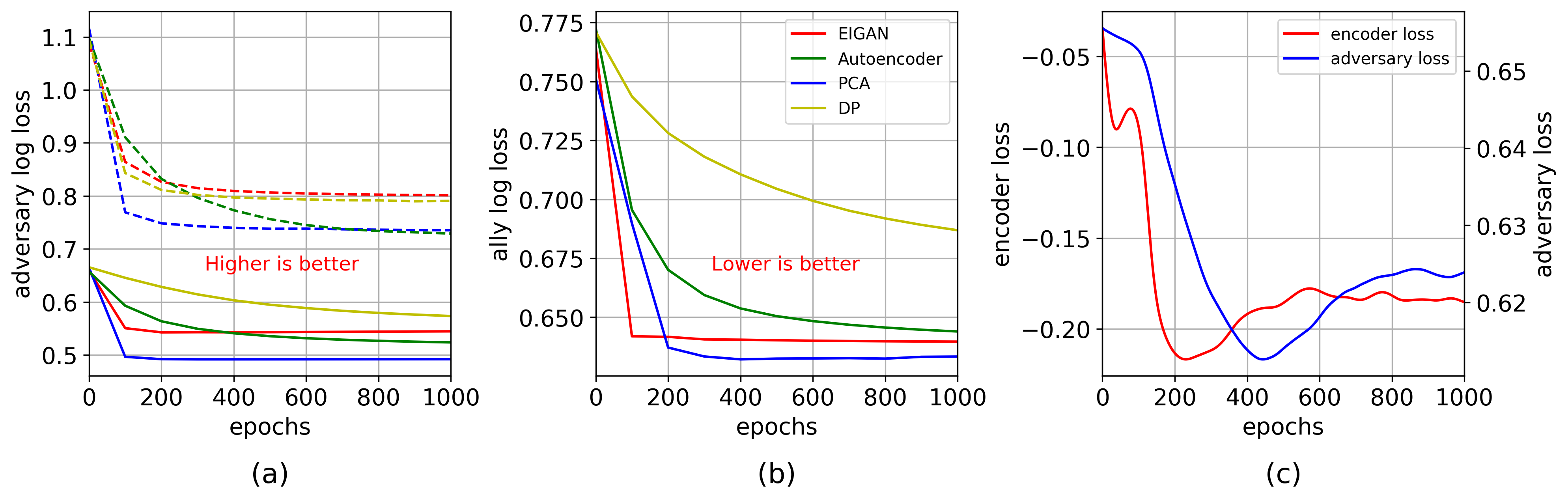}}
\vskip -0.1in
\caption{\small{Predictivity and privacy comparison between EIGAN and the baselines across one ally and two adversaries on the Titanic dataset. (a) On one of the adversary objectives (gender prediction, solid lines) EIGAN matches DP's performance (by design of the experiment, as determined by the selection of the DP $\epsilon$ parameter), but in this case it does not match the other adversary prediction (passenger class prediction, dashed lines), which could be matched for another value of $\epsilon$. (b) On the ally objective (survival prediction), EIGAN achieves marginal improvement over the the baseline Autoencoder. (c) EIGAN training converges after initial oscillations corresponding to the minimax game.}}
\label{fig:titanic_comparison}
\end{center}
\vskip -0.2in
\end{figure*}

\subsection{Comparison on Titanic dataset}
\label{suppsub:eigan-titanic}

For completeness, we also evaluate EIGAN algorithm on another dataset, Titanic, which consists of data listing the details of roughly 800 of the passengers that were onboard the Titanic ship. This experiment aims at understanding the convergence behaviour of EIGAN under limited training data. 

\begin{table}[h!]
\begin{center} \small{
\begin{tabular}{|l|c|c|c|r|}
\hline
\multirow{2}{*}{\textbf{Algorithm}} & \textbf{Ally} & \textbf{Adversary 1} & \textbf{Adversary 2} \\
  & (Survival) & (Gender) & (P-Class)\\
\hline
\hline
Autoencoder    & \textbf{0.6333} & 0.4918 & 0.7351 \\
PCA & 0.6439 & 0.5236 & 0.7289\\
DP    & 0.6869  & \textbf{0.5733} & 0.7904 \\
EIGAN    & \textbf{0.6396} & 0.5444 & \textbf{0.8011}         \\
\hline
\end{tabular} }
\end{center}
\caption{\small{Comparison of log-loss achieved on the test set between the algorithms for the Titanic dataset. EIGAN matches autoencoder on the ally and performs slightly better than DP on adversary 2, while slightly worse on adversary 1.}}
\vskip -0.2in
\label{table:titanic_comparison}
\end{table}

Similar to result on MIMIC-III from Fig.~\ref{fig:mimic_comparison}, Fig.~\ref{fig:titanic_comparison} (a) shows that while EIGAN is able to perform as well or nearly as well as any of the baselines on adversary obfuscation, (b) it obtains the best predictivity on ally objective. (c) shows that the training reaches a steady-state. 

Table~\ref{table:titanic_comparison} summarizes the loss-values of the trained allies/adversaries on encoded data using different techniques. It can be seen that while EIGAN is able to match DP's performance on adversary 2, it performs marginally worse than it on adversary 1, while having a considerable gain on the corresponding ally.

\begin{figure*}[h!]
\begin{center}
\centerline{\includegraphics[width=0.85\textwidth]{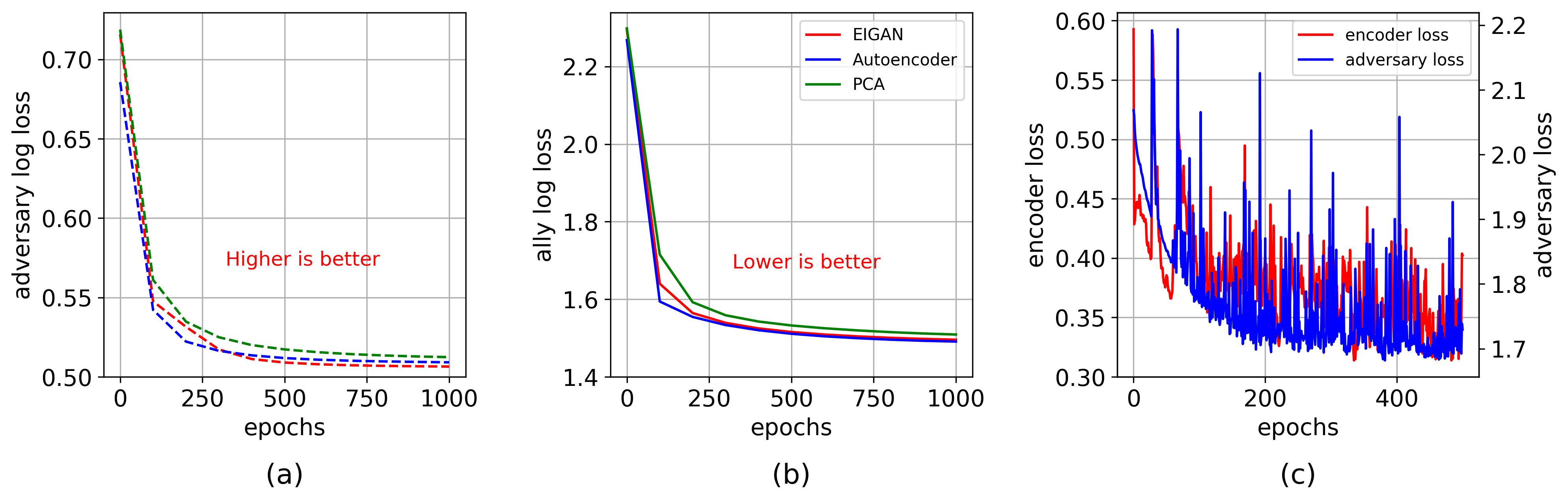}}
    \captionof{figure}{\small{Comparison across one ally and two adversaries on the MNIST dataset. The (a) adversary objective (odd-even prediction, a binary classification with virtually identical trends) converge to roughly the same loss for each algorithm, and (b) ally objective (digit prediction, 10-class classification). With dependencies (in particular, partial overlaps) between the ally and adversary objectives, EIGAN training in (c) is unable to fully converge, consistent with Prop.~\ref{prop:nonuniform}.}}
    \label{fig:mnist_comparison}
\end{center}
\vskip -0.2in
\end{figure*}

\subsection{Comparison on MNIST}
\label{suppsub:eigan-mnist}
We conduct an additional experiment on the MNIST dataset of handwritten digits to validate the findings in Prop. \ref{prop:uniform}\&\ref{prop:nonuniform} when dependencies exist between the ally and adversary objectives. In this case, we use digit recognition (0-9) as the ally objective and even vs odd as adversary objective, which exhibits a clear dependence because if someone could recover the digit (ally objective), then inferring odd-vs-even (adversary objective) becomes trivial. Formally, referring to the propositions, we have $\mathcal{Y}_{\text{odd}} = \mathcal{Y}_1 + \mathcal{Y}_3 + \cdots + \mathcal{Y}_9$ and $\mathcal{Y}_{\text{even}} = 1 - \mathcal{Y}_{\text{odd}}$ where $\mathcal{Y}_{(\cdot)}$ is the true probability distribution on the labels and thus can be added. Similarly, $\hat{Y}_{\text{odd}} = \hat{Y}_1 + \hat{Y}_3 + \cdots + \hat{Y}_9$ and $\hat{Y}_{\text{even}} = 1 - \hat{Y}_{\text{odd}}$, where $\hat{Y}_{(\cdot)}$ are probabilities of correct predictions. Prop.~\ref{prop:nonuniform} follows when we substitute these in \eqref{eq:prop3_ent}, i.e. the adversary is not forced to a follow uniform distribution if sufficient weight is given to the ally.

Fig.~\ref{fig:mnist_comparison} shows the result of this experiment, where the weights of the allies and adversaries are set equal. (a) shows that the adversary is not able to achieve any separation from the Autoencoder or PCA. Observing (c), we realize that the training process does not reach a steady state-convergence point, consistent with the propositions.

\begin{figure}[h!]
\begin{center}
\centerline{\includegraphics[width=0.6\columnwidth]{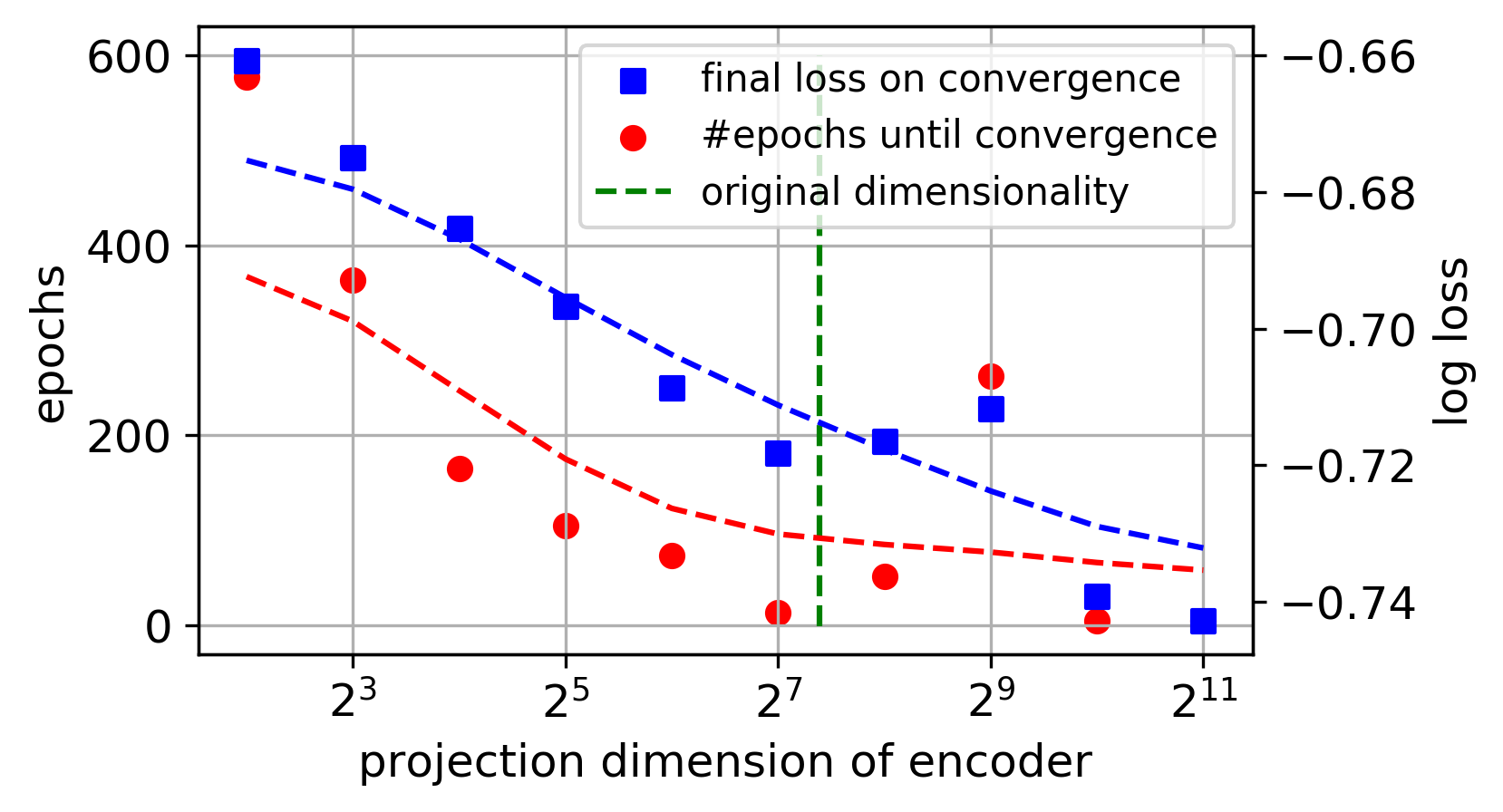}}
\vskip -0.1in
\caption{\small{Effect of EIGAN's encoding dimension space on the number of training epochs required to reach within 1\% of training loss convergence (left axis) and the achieved final testing loss (right axis) for MIMIC-III. The achieved loss decreases sharply as the dimension increases, emphasizing a tradeoff between model quality and the memory needed for the encoded data. In fact, beyond the right end of the X-axis value, the model runs out of memory on our high performance machine. (Dashed curves are fit using weighted moving averages.)}}
\label{fig:dim_comparison}
\end{center}
\vskip -0.2in
\end{figure}

\begin{figure}[h!]
\begin{center}
\centerline{\includegraphics[width=0.6\columnwidth]{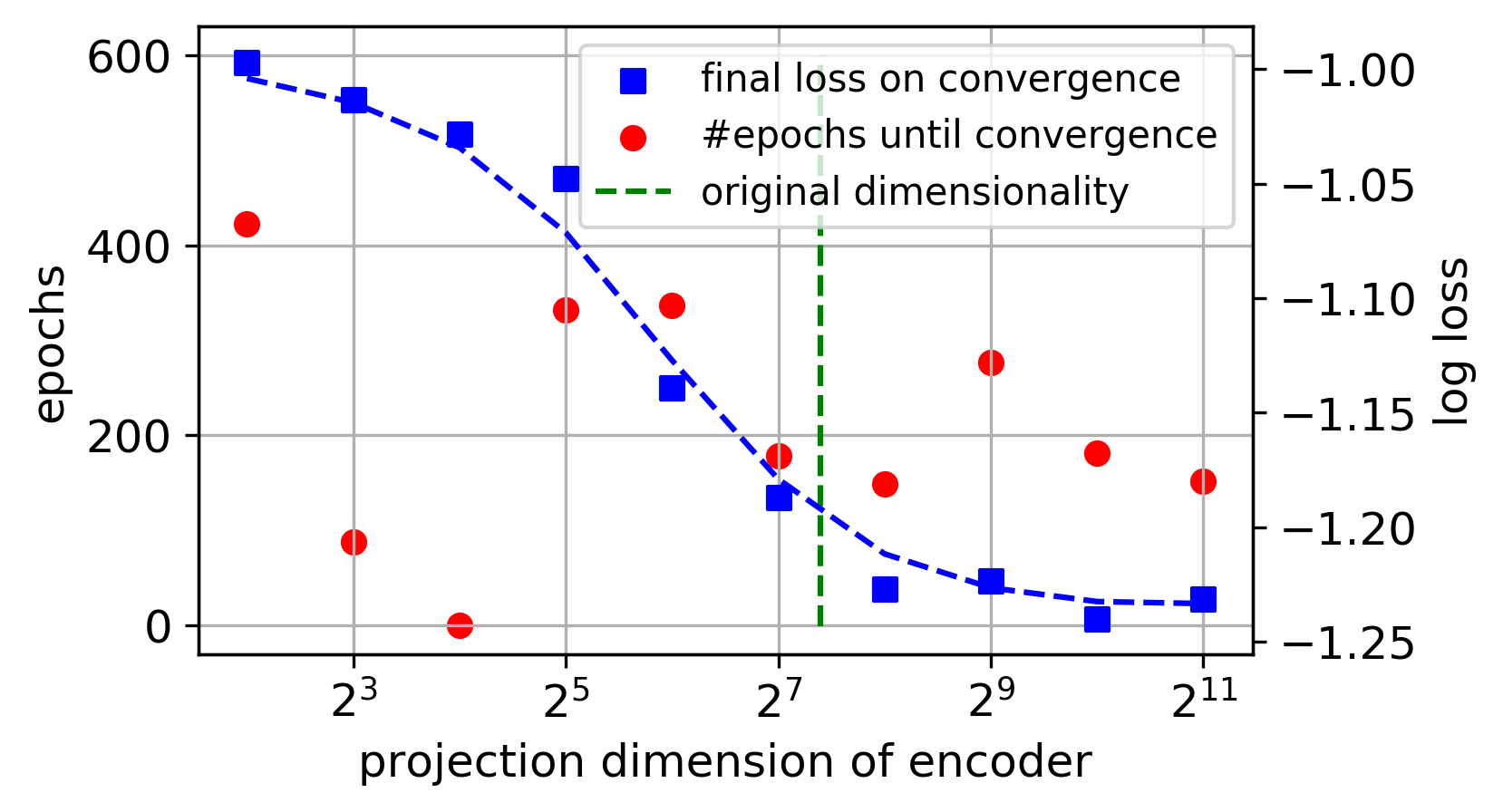}}
\vskip -0.1in
\caption{\small{Effect of EIGAN's encoding dimension space on the number of training epochs required to reach within 1\% of training loss convergence (left axis) and the achieved final testing loss (right axis) for the Titanic dataset. The achieved loss decreases sharply as the dimension increases, emphasizing a tradeoff between model quality and required memory. (The dashed curve is fit using a weighted moving average.)}}
\label{fig:dim_comparison_titanic}
\end{center}
\vskip -0.2in
\end{figure}

\subsection{Varying Encoder Dimensionality}
\label{suppsub:eigan-encoder-dim}

Under \textit{varying system dimensions}, in Sec.~\ref{ssec:c-eigan-expt}, we also discussed the effect of varying the encoder dimensionality.  Fig.~\ref{fig:dim_comparison} depicts the results for the MIMIC-III dataset while Fig.~\ref{fig:dim_comparison_titanic} depicts the result of a similar experiment on Titanic dataset. In the two experiments, as the encoder output dimension $l$ is increased, we observe that the training mostly requires fewer epochs to converge and is able to achieve a lower encoder testing loss. This could be explained by the fact that larger networks (i.e. more number of trainable parameters) have more degrees of freedom in training. Interestingly, while there is some variation, the test loss continues to decrease beyond $d$, the original dimension of the data samples, i.e., when $l \geq d$. The relevant consideration with EIGAN, then, appears to be the tradeoff between encoding quality, as measured by the encoding space dimension, and the memory required for training the encoder, which increases with the dimension of the encoder.

\section{Additional D-EIGAN Experiments}
\label{supp:d-eigan-expt}

\begin{figure}[h!]
\begin{center}
\centerline{\includegraphics[width=0.8\columnwidth]{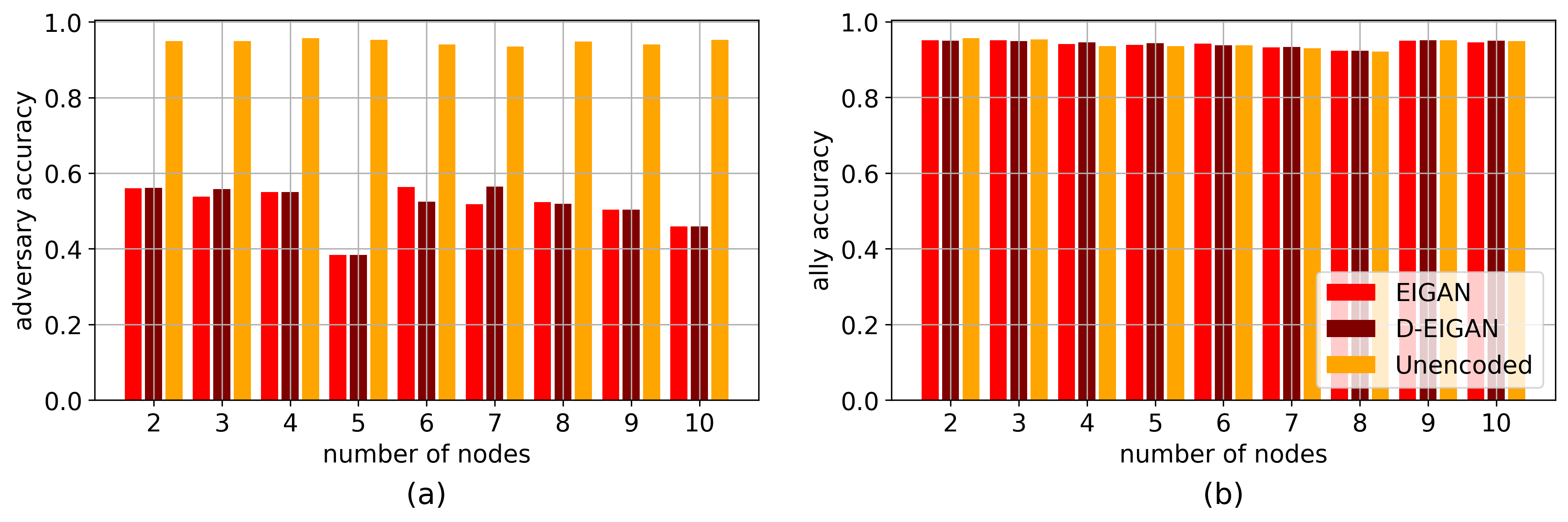}}
\vspace{-0.1in}
\caption{\small{Comparison of (a) adversary and (b) ally performance using synthetic Gaussian data while increasing the number of nodes and sharing all the model weights ($\phi=1$) after every minibatch ($\delta = 1$) during federated training. The distribution of data is i.i.d. across the nodes, which is obtained by generating Gaussian data with constant mean and variance across nodes. It can be observed that EIGAN and D-EIGAN converge to similar performances regardless of the number of nodes.}}
\label{fig:distributed_num_nodes_2}
\end{center}
\vskip -0.2in
\end{figure}

\subsection{Varying Number of Nodes}
\label{suppsub:d-eigan-numnodes}

Fig.~\ref{fig:distributed_num_nodes} from Sec.~\ref{ssec:d-eigan-expt} presented the effect of varying the number of nodes on D-EIGAN performance when the nodes have non-i.i.d data distributions. Fig.~\ref{fig:distributed_num_nodes_2} shows the result of the experiment when the nodes instead have i.i.d data. We observe that the performance of the ally and adversary remains reasonably constant (and similar to EIGAN) as we increase the number of nodes under D-EIGAN. From the two experiments, we can conclude that D-EIGAN can readily extend to scenarios where data is distributed over larger number of nodes without sacrificing the performance on ally and adversary objectives.

\begin{figure}[h!]
\begin{center}
\centerline{\includegraphics[width=0.6\columnwidth]{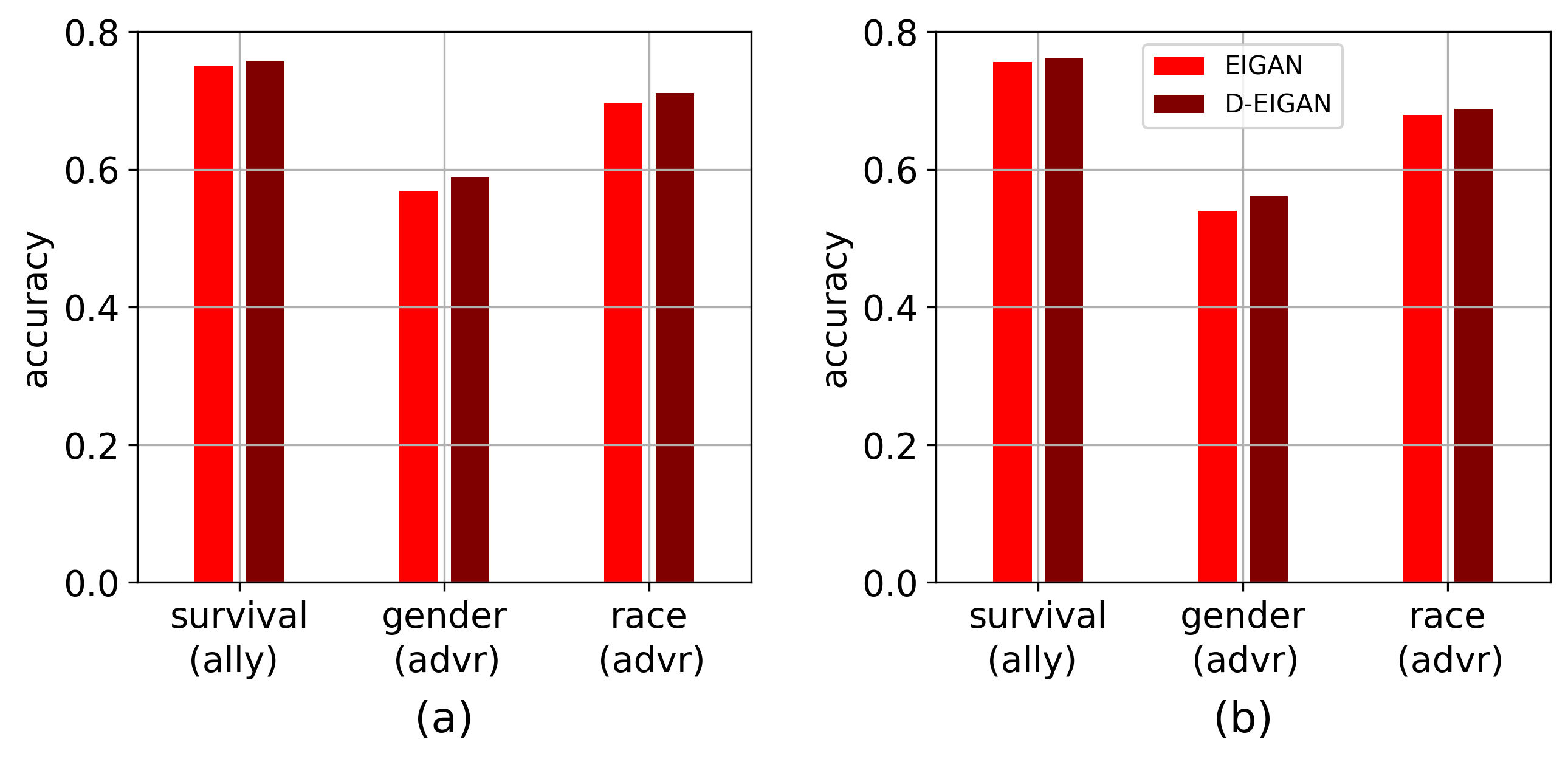}}
    \vskip -0.15in
    \caption{\small{Comparison of distributed ($K = 2$ nodes) EIGAN with centralized EIGAN. Survival is the ally objective, and gender and race are the chosen adversary objectives for the experiment. (a) Training of distributed EIGAN involves same adversary objectives, i.e., obfuscating gender and race across the both the nodes. (b) Each node has a different adversary objective, while they share the same ally objective.}}
    \label{fig:distributed_eigan_mimic_2}
\end{center}
\vskip -0.2in
\end{figure}

\subsection{Varying Objectives across Nodes}
\label{suppsub:d-eigan-var-obj}

Fig.~\ref{fig:distributed_eigan_mimic_2} is an addition to results presented in Fig.~\ref{fig:distributed_eigan_mimic} from Sec.~\ref{ssec:d-eigan-expt}. In this experiment we consider the same dataset with the same set of ally and adversary objectives, but in this case over i.i.d datasets on 2 nodes instead of non-i.i.d datasets on 10 nodes. Survival is the ally objective while, gender and race are adversary objectives. We consider two scenarios: in (a) gender and race are common adversaries across the two nodes, while in (b) gender is the adversary on one node and race is the adversary on the other node (i.e. different objectives on different nodes). We observe the performance of D-EIGAN is comparable to that of EIGAN in the case where the data is centralized. This observed behavior, i.e., that a privacy and/or predictivity objective at one node is adopted across all the encoders, is consistent with Prop.~\ref{prop:dist_obj}.

\begin{figure}[h!]
\begin{center}
    \centerline{\includegraphics[width=0.6\columnwidth]{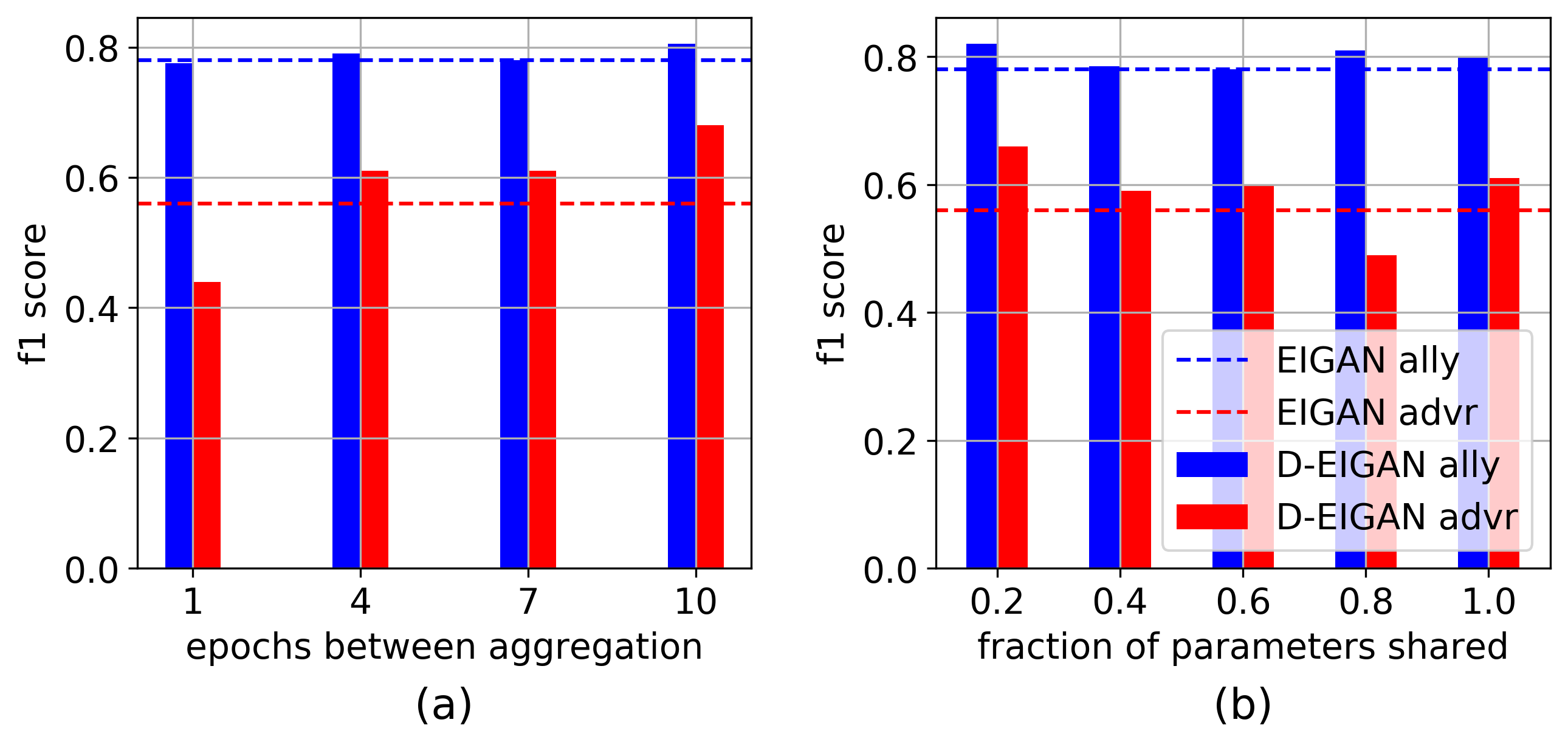}}
    \vskip -0.15in
    \caption{\small{Effect of varying (a) frequency of sync ($\delta$, measured in terms of number of epochs between parameter sharing) and (b) fraction of parameters uploaded/downloaded ($\phi$) on a distributed implementation consisting of $K = 2$ nodes. The results shows that as the frequency of sync/fraction of parameters shared increases, the performance of the system on hiding the sensitive variable is increased considerably, while there is little effect on the ally convergence.}}
    \label{fig:distributed_control_params_2}
\end{center}
\vskip -0.3in
\end{figure}

\subsection{Varying Synchronization parameters}
\label{suppsub:d-eigan-sync-params}

This section extends the results discussed in Sec.~\ref{ssec:d-eigan-expt} under \textit{varying synchronization parameters}. To understand the effect of fractional parameter sharing, we evaluate it on i.i.d. datasets over 2 nodes using the synthetic Gaussian dataset in Sec.~\ref{ssec:d-eigan-expt}. The data has unbalanced classes, so we compare f1-score instead of accuracy. The result is shown in Fig.~\ref{fig:distributed_control_params_2}: We see that unlike the trend on the non-i.i.d. case, there is no visible benefit of sharing only a fraction of parameters, as seen in Fig.~\ref{fig:distributed_control_params_2}(b). Similarly, in (a) it can be observed that performance over the adversary degrades as the frequency of sync is decreased, i.e., number of epochs between aggregation in increased. Hence, the properties observed upon having non-i.i.d data distributions in Sec.~\ref{ssec:d-eigan-expt} do not hold when handling i.i.d data. This is because the reduction in model bias is not desirable in the case of i.i.d.

\end{document}